\documentclass[journal]{IEEEtran}
\usepackage{graphicx}
\usepackage{amsmath}
\usepackage{amsfonts}
\usepackage{amssymb}
\usepackage{bm}
\usepackage{subfigure}
\usepackage{threeparttable,array}
\usepackage{color}
\usepackage[ruled,linesnumbered]{algorithm2e}
\usepackage{amsthm}

\newtheorem{theorem}{Theorem}
\newtheorem{remark}{Remark}
\newtheorem{definition}{Definition}
\newtheorem{lemma}{Lemma}

\newtheorem{corollary}{Corollary}
\newtheorem{assumption}{Assumption}

\ifCLASSINFOpdf
\else
\fi
\begin{document}
\graphicspath{{figures/}}
%
\title{A Novel Vector-Field-Based Motion Planning Algorithm for 3D Nonholonomic Robots}
%
%
%

\author{Xiaodong~He, Weijia Yao, Zhiyong Sun and
        Zhongkui~Li
\thanks{X. He and Z. Li are with the State Key Laboratory for Turbulence and Complex Systems, Department of Mechanics and Engineering Science, College of Engineering, Peking University, Beijing 100871, China (e-mail: hxdupc@pku.edu.cn; zhongkli@pku.edu.cn)

W. Yao is with the Institute of Engineering and Technology (ENTEG), University of Groningen, 9747 AG Groningen, The Netherlands (email: w.yao@rug.nl)

Z. Sun is with the Department of Electrical Engineering, Eindhoven University of Technology, 5612 AZ Eindhoven, The Netherlands (email: sun.zhiyong.cn@gmail.com)}
}

\maketitle

\begin{abstract}
This paper focuses on the motion planning for mobile robots in 3D, which are modelled by 6-DOF rigid body systems with nonholonomic kinematics constraints. We not only specify the target position, but also bring in the requirement of the heading direction at the terminal time, which gives rise to a new and more challenging 3D motion planning problem. The proposed planning algorithm involves a novel velocity vector field (VF) over the workspace, and by following the VF, the robot can be navigated to the destination with the specified heading direction. In order to circumvent potential collisions with obstacles and other robots, a composite VF is designed by composing the navigation VF and an additional VF tangential to the boundary of the dangerous area. Moreover, we propose a priority-based algorithm to deal with the motion coupling issue among multiple robots. Finally, numerical simulations are conducted to verify the theoretical results.
\end{abstract}


%
\IEEEpeerreviewmaketitle

\section{Introduction}

As a fundamental problem in robotics and control, motion planning refers to finding a path or trajectory which guides robots from an initial position to a goal position, without collisions with obstacles or other robots. Motion planning algorithms have been employed in a wide range of robotic systems, e.g., planar mobile robots \cite{Wang2020EB_RRT,Li2021Efficient,Zhao2021Pareto}, marine vessels \cite{Zhou2022A_guidance,Hassani2018Path,Pedersen2012Marine}, and quadrotors \cite{Zhou2021Raptor,Zhou2021Ego_planner,Tordesillas2022MADER}. Generally, to reduce weight and cost, robotic systems are often designed with fewer actuators than the degrees of freedom (DOF),
resulting in underactuated systems. From the kinematic level of mobile robots in 3D, a typical form of underactuation is that the robot modelled by a 6-DOF rigid body has no sway and heave velocity, while only one surge velocity and three angular velocities serve as control inputs. Such a system model is referred to as a 3D nonholonomic rigid body, and can be utilized to describe the kinematics of various robotic systems, such as fixed-wing UAVs \cite{Wang2019Coordinated,Levin2019Real_time,Fari2020Addressing,He_arXiv_roto} and autonomous underwater vehicles \cite{Egeland1996Feedback,Li2013Finite,Li2017Receding,He2022Exponential}.

The motion model of a 3D nonholonomic mobile robot can be described by a nonholonomic constrained rigid body with six DOFs. Resulting from the nonholonomic constraints, the heading direction of the robot is restricted and always points to the $x$-axis of the body-fixed frame \cite{Beard2012Small}. Thus, we are motivated to take into account not only the goal position but also the heading direction of the robot in motion planning. For instance, in real-world scenarios, the heading directions of multiple missiles are typically specified in the terminal guidance so as to realize a better performance of coordinated attack. Similarly, in surveillance tasks performed by multi-UAV systems, the final orientation of each UAV should point to a certain direction to obtain the largest overall surveillance area. Therefore, from the perspectives of theory and practice, it is of great importance to investigate the motion planning problem with a specified position and a heading direction.

However, it is not trivial to simultaneously plan the position and heading direction of 3D nonholonomic robots, and the main challenge arises from the underactuation caused by the nonholonomic constraints. On one hand, since the linear velocity is restricted to the heading direction, one should adjust the heading direction towards the destination such that the robot can move towards the goal position. On the other hand, the heading direction is not only employed for reaching the destination, but also has to satisfy the requirement of terminal direction, which demonstrates the underactuated characteristics of 3D nonholonomic robots. Besides, compared to the 2D case which has only one rotational DOF, the rotation control in 3D is more complicated due to the coupling among three different rotational DOFs, leading to a more challenging motion planning problem.

In the literature, there exist several commonly-used methodologies for motion planning, including the roadmap method \cite{Kavraki1996Probabilistic,Bhattacharya2008Roadmap,Lehner2018Repetition}, the cell decomposition approach \cite{Cai2009Information,Zhang2008Efficient,Cowlagi2012Multiresolution} and the sampling-based algorithm \cite{Karaman2011Sampling,Jaillet2010Sampling,Oh2021Chance-Constrained}. But these methods cannot be directly applied to the nonholonomic robotic systems. Although certain optimization-based algorithms can handle the nonholonomic constraints, such as \cite{Hussein2008Optimal,Hausler2016Energy,Cichella2021Optimal,Zhao2022Scalable,Li2021Optimal,Bloch2021Dynamic}, the feasibility of the optimization problem cannot be guaranteed or it suffers from heavy computational burden once the nonholonomic constraints are taken into consideration. Besides, the control inputs derived from these optimization-based algorithms are open-loop, relying only on time, thereby not robust to disturbances.

To overcome the above-mentioned limitations, we utilize the vector field (VF) method to solve the motion planning problem for the 3D nonholonomic robots. Specifically, a velocity VF is defined over the workspace and the integral curve of the VF converges to the goal point. Compared to the optimization-based method, the VF motion planning specifies a reactive feedback in the sense that it is only related to the current state while path replanning is not required. Furthermore, the VF directly specifies the heading direction at each point, which provides a reference to adjust the attitude of the robot so as to follow such a heading direction. In this way, the challenge caused by the nonholonomic constraints can be suitably handled via the VF method. Although several works have employed VF in the motion planning problem, such as \cite{Lindemann2009Simple,Panagou2017A_distributed,Marchidan2020Collision,Kapitanyuk2018A_Guiding,Yao2022Guiding,He_arXiv_simul}, most of them focus on particle agents in 2D instead of rigid bodies in 3D. Additionally, the heading direction of the nonholonomic robot is rarely considered in the existing VF-based motion planning results.

Therefore, in this paper, we investigate the motion planning problem of 3D nonholonomic robots via the VF method, where the terminal heading direction is taken into account. Moreover, the robot is modelled by a 6-DOF nonholonomic constrained rigid body rather than a particle agent. The contributions of this paper are given below.
\begin{enumerate}
  \item Concerning the nonholonomic robot moving in an obstacle-free environment, we design a navigation VF (N-VF) which converges to any desired target position. More importantly, the N-VF points to a specified direction at the destination. Thus, by moving along the N-VF, the nonholonomic robot can reach the target position with the specified heading direction.
  \item An obstacle avoidance VF (OA-VF) is proposed for motion planning in an obstacle-cluttered environment. The shape of the obstacle  is described by the level surface of an implicit function, and the obstacle can be either static or moving. The OA-VF is composed of the N-VF and an additional  VF tangential to the obstacle surface, and the composite OA-VF is free of any singularity.
  \item For motion planning of multiple nonholonomic robots, a collision avoidance VF (CA-VF) is proposed to evade the collisions among nonholonomic robots. In order to reduce the motion coupling in collision avoidance, we propose a priority-based algorithm where the movement of higher prioritized nonholonomic robots is independent of the lower ones, so as to achieve the motion decoupling.
  \item To apply the proposed VF to the 3D rigid body, we design another two auxiliary VFs which are always orthogonal to the N/OA/CA-VF. Moreover, such two auxiliary VFs are also orthogonal to each other. Therefore, these three VFs construct the basis of a body-fixed frame in $\mathbb{R}^3$, which can be further utilized as a reference for the attitude control of 3D nonholonomic robots.
\end{enumerate}

To the best of our knowledge, it is the first time that the motion planning problem of 6-DOF nonholonomic rigid-body robots in 3D is investigated taking into account both specified terminal positions and heading directions.

This paper is organized as follows. Section~\ref{sec_prelimi} provides the preliminaries and formulates the problems. The motion planning algorithms in obstacle-free and obstacle-cluttered environments are proposed in Section~\ref{sec_VF} and Section~\ref{sec_OA_VF}, respectively. Section~\ref{sec_CA_VF} provides the cooperative motion planning algorithm for multiple nonholonomic robots. Numerical simulation results are given in Section~\ref{sec_sim}, followed by Section~\ref{sec_conclu}, which concludes the paper.

\emph{Notations}: Let $\dot{f}$ denote the time derivative $\frac{{\rm d}f(t)}{{\rm d}t}$ of any differentiable function $f(t)$. The distance between a point $\bar{\bm{p}}\in\mathbb{R}^n$ and a nonempty set $\bm{S}\subseteq\mathbb{R}^n$ is defined by $\verb"dist"(\bm{\bar{\bm{p}}},\bm{S}):=\inf\{\|\bm{p}-\bar{\bm{p}}\| : \bm{p}\in\bm{S}\}$, where $\|\cdot\|$ denotes the Euclidean norm. The distance between two nonempty sets $\bm{A}\subseteq\mathbb{R}^n$ and $\bm{B}\subseteq\mathbb{R}^n$ is defined by $\verb"dist"(\bm{A},\bm{B}):=\inf\{\|\bm{a}-\bm{b}\| : \bm{a}\in\bm{A},\bm{b}\in\bm{B}\}$. The notation $\bm{I}$ represents an identity matrix of suitable dimensions. The notations $\verb"det"$ and $\verb"ker"$ represent the determinant and kernel of a matrix, respectively. The words ``with respect to" is abbreviated as ``w.r.t".

\section{Preliminaries and Problem Formulation} \label{sec_prelimi}

\subsection{Nonholonomic robots modelling}

Consider a swarm of nonholonomic robots labelled by $\mathcal{I}_{\mathcal{V}}=\{1,\cdots,N\}$, and each robot is modelled by a 6-DOF rigid body moving in the 3D Euclidean space $\mathbb{R}^3$. Let $\bm{\mathcal{F}}_{\rm e}$ denote the earth-fixed frame and let $\bm{\mathcal{F}}_{\rm b}$ denote the body-fixed frame, which is attached to the center of mass of the robot. The position of the $i$-th nonholonomic robot ($i\in\mathcal{I}_\mathcal{V}$) in $\bm{\mathcal{F}}_{\rm e}$ is described by a vector $\bm{p}_i=[x_i\ \ y_i\ \ z_i]^{\rm T}\in\mathbb{R}^3$, while the attitude is specified by a rotation matrix $\bm{R}_i\in{\rm SO(3)}:=\{\bm{R}\in\mathbb{R}^{3\times 3}:\bm{R}^{\rm T}\bm{R}=\bm{I},\verb"det"({\bm{R}})=1\}$, which depicts the rotation of $\bm{\mathcal{F}}_{\rm b}$ relative to $\bm{\mathcal{F}}_{\rm e}$. The $i$-th nonholonomic robot's angular velocity and linear velocity are denoted by $\bm{\Omega}_i=[\Omega_{xi} \ \Omega_{yi}\ \ \Omega_{zi}]^{\rm T}\in\mathbb{R}^3$ and $\bm{v}_i=[v_{xi}\ \ v_{yi}\ \ v_{zi}]^{\rm T}\in\mathbb{R}^3$, respectively, which are both provided in the body-fixed frame $\bm{\mathcal{F}}_{\rm b}$. For $\forall\bm{a}=[a_x\ \ a_y\ \ a_z]^{\rm T}\in\mathbb{R}^3$, we define the linear map ``$\wedge$" as follows
\begin{equation}\label{eq_hat_map}
  \bm{a}^{\wedge}=
  \begin{bmatrix}
     0   & -a_z & a_y \\
    a_z  &  0   & -a_x \\
    -a_y & a_x  & 0
  \end{bmatrix}\in\mathfrak{so}(3),
\end{equation}
where $\mathfrak{so}(3):=\{\bm{S}\in\mathbb{R}^{3\times 3}:\bm{S}^{\rm T}=-\bm{S}\}$ denotes the Lie algebra of ${\rm SO(3)}$. Then, for $\forall\bm{a},\bm{b}\in\mathbb{R}^3$, there holds $\bm{a}^{\wedge}\bm{b}=\bm{a}\times\bm{b}$, where ``$\times$" denotes the cross product. Thus, the kinematics of the $i$-th nonholonomic robot can be given by
\begin{subequations}\label{eq_kine_UVA}
  \begin{align}
    \dot{\bm{R}}_i&=\bm{R}_i\bm{\Omega}^{\wedge}_i, \\
    \dot{\bm{p}}_i&=\bm{R}_i\bm{v}_i, \label{eq_kine_trans}
  \end{align}
\end{subequations}
which describes the rotation and translation kinematics, respectively. Note that the motion of the nonholonomic robot is restricted by the nonholonomic constraints, so that the linear velocities along the $y$-axis and $z$-axis of the body-fixed frame $\bm{\mathcal{F}}_{\rm b}$ are both zero, that is, $v_{yi}=v_{zi}=0$. Hence, the nonholonomic robot with 6 DOFs is controlled by only 4 inputs, i.e., $\Omega_{xi},\Omega_{yi},\Omega_{zi},v_{xi}$, demonstrating that the nonholonomic robot is a typical underactuated system.

Assume that the dangerous area $\bm{\mathcal{C}}_{oi}$ of the $i$-th nonholonomic robot is a sphere given by
\begin{equation}\label{eq_C_ai}
  \bm{\mathcal{C}}_{oi}=\{\bm{p}\in\mathbb{R}^3:\bm{\Psi}_i(\bm{p};\bm{p}_i)=\|\bm{p}-\bm{p}_i\|<r_c\},
\end{equation}
where $\bm{p}_i$ is the position of the $i$-th nonholonomic robot and $r_c$ is the radius of the sphere. If the dangerous area $\bm{\mathcal{C}}_{oi}$ is entered by other nonholonomic robots, a potential collision is supposed to occur. Similarly, we can define the reactive area $\bm{\mathcal{C}}_{ri}$ of the $i$-th nonholonomic robot by
\begin{equation}\label{eq_C_ri}
  \bm{\mathcal{C}}_{ri}=\{\bm{p}\in\mathbb{R}^3:r_c\leq\bm{\Psi}_i(\bm{p};\bm{p}_i)\leq r_d\},
\end{equation}
where $r_d$ is the communication or detection range of the nonholonomic robot. Thus, the nonholonomic robot should take actions to avoid collisions once it enters the reactive area $\bm{\mathcal{C}}_{ri}$ of other nonholonomic robots.

\subsection{Obstacles Description}

Consider a finite set of obstacles in the environment, and each of them is described  by a continuous function $\bm{\Upsilon}_i(\bm{p};\bm{p}_{oi}):\mathbb{R}^3\to\mathbb{R}$, where $\bm{p}=[x\ y\ z]^{\rm T}$ is the position variable and $\bm{p}_{oi}=[x_{oi}\ y_{oi}\ z_{oi}]^{\rm T}$ is either a constant or a time-varying vector in $\mathbb{R}^3$, $i\in\mathcal{I}_\mathcal{O}=\{1,\cdots,M\}$, and $M$ is the number of the obstacles. Assume that $\bm{\Upsilon}_i(\bm{p};\bm{p}_{oi})$ has continuous first-order derivatives and increases monotonically w.r.t $\|\bm{p}-\bm{p}_{oi}\|$. Moreover, the level surfaces of $\bm{\Upsilon}_i(\bm{p};\bm{p}_{oi})$, that is, $\bm{\Upsilon}_i(\bm{p};\bm{p}_{oi})=\bar{c}_i$ (where $\bar{c}_i$ is a positive constant) can enclose a region. Then, we utilize the equation $\bm{\Upsilon}_i(\bm{p};\bm{p}_{oi})=1$ to describe the surface of the obstacle, whose center is located at $\bm{p}_{oi}$. For instance, the surface $\bm{\Upsilon}_i(\bm{p};\bm{p}_{oi})=\frac{(x-x_{oi})^2}{a^2}+\frac{(y-y_{oi})^2}{b^2}+\frac{(z-z_{oi})^2}{c^2}=1$ represents an ellipsoid centered at $(x_{oi},y_{oi},z_{oi})$ with principal semiaxes of lengths $a,b,c$. The obstacle can be static or moving, which depends on the motion of $\bm{p}_{oi}$, and it is assumed that the obstacle's velocity $\dot{\bm{p}}_{oi}$ is known.

With the help of $\bm{\Upsilon}_i(\bm{p};\bm{p}_{oi})$, we can define the obstacle area $\bm{\mathcal{A}}_{oi}$ and the reactive area $\bm{\mathcal{A}}_{ri}$ as given below.
\begin{align}
  \bm{\mathcal{A}}_{oi} & = \{\bm{p}\in\mathbb{R}^3: \bm{\Upsilon}_i(\bm{p};\bm{p}_{oi})<1 \}, \label{eq_A_o} \\
  \bm{\mathcal{A}}_{ri} & = \{\bm{p}\in\mathbb{R}^3: 1\le\bm{\Upsilon}_i(\bm{p};\bm{p}_{oi})\le\bar{c}_i \}. \label{eq_A_r}
\end{align}
The obstacle area $\bm{\mathcal{A}}_{oi}$ is the region where the obstacle occupies, and collision occurs if the nonholonomic robot enters $\bm{\mathcal{A}}_{oi}$. Once in the reactive area $\bm{\mathcal{A}}_{ri}$, the nonholonomic robot can sense the obstacle and needs to be reactive to the obstacle such that a potential collision can be avoided. Note that the size of the reactive area $\bm{\mathcal{A}}_{ri}$ can be controlled by the constant $\bar{c}_i$.
The following standing assumptions are imposed.

\begin{assumption}\label{assum_p_d}
  There holds that $\bm{p}_{dj}\notin\bigcup_{i\in\mathcal{I}_\mathcal{O}}(\bm{\mathcal{A}}_{oi}\cup\bm{\mathcal{A}}_{ri})$ for all time instant $t>0$ and all $j\in\mathcal{I}_\mathcal{V}$, where $\bm{p}_{dj}\in\mathbb{R}^3$ denotes the target position of the $j$-th nonholonomic robot.
\end{assumption}

\begin{assumption}\label{assum_dist_ob}
  There holds $\verb"dist"(\bm{\mathcal{A}}_{ri},\bm{\mathcal{A}}_{rj})>0$ for all $i\ne j\in\mathcal{I}_{\mathcal{O}}$ and $t>0$.
\end{assumption}

\begin{assumption} \label{assum_compact_ob}
    The reactive areas $\bm{\mathcal{C}}_{ri}$ and $\bm{\mathcal{A}}_{rj}$ are compact for all $i \in \mathcal{I}_\mathcal{V}$ and $j \in \mathcal{I}_\mathcal{O}$.
\end{assumption}

Assumption~\ref{assum_p_d} means that the target position of each robot cannot be covered by obstacles. Assumption~\ref{assum_dist_ob} implies that any two obstacles are sufficiently far away such that their reactive areas do not overlap\footnote{If two obstacles are too close such that this assumption is violated, then these two obstacles can be regarded as one big obstacle such that this assumption holds.}. Assumption~\ref{assum_compact_ob} stipulates that robots and obstacles are of finite sizes (i.e., bounded).

\subsection{Problem Formulation}\label{subsec_pf}

\textbf{Motion Planning Problem}: Consider a swarm of nonholonomic robots labelled by $i\in\mathcal{I}_\mathcal{V}$, and the kinematics of each robot is described by (\ref{eq_kine_UVA}). Let $\bm{p}_{di}\in\mathbb{R}^3$ denote a target position and let the unit vector $\bm{e}_{di}\in\mathbb{R}^3$ denote a desired heading direction of the $i$-th nonholonomic robot. Then, we design the angular velocity $\bm{\Omega}^{\wedge}_i$ and the linear velocity $\bm{v}_{i}$ in (\ref{eq_kine_UVA}), such that
\begin{enumerate}
  \item $\lim_{t\to\infty}\|\bm{p}_i-\bm{p}_{di}\|=0$ and $\lim_{t\to\infty} \left\|\bm{R}_i\frac{\bm{v}_i}{\|\bm{v}_i\|}-\bm{e}_{di} \right\|=0$ for all $i\in\mathcal{I}_\mathcal{V}$;
  \item $\verb"dist"(\bm{\mathcal{C}}_{oi},\bm{\mathcal{A}}_{oj})>0$ for all $i\in\mathcal{I}_\mathcal{V}$, $j\in\mathcal{I}_{\mathcal{O}}$ and $t>0$;
  \item $\verb"dist"(\bm{\mathcal{C}}_{oi},\bm{\mathcal{C}}_{oj})>0$ for all $i\ne j\in\mathcal{I}_{\mathcal{V}}$ and $t>0$.
\end{enumerate}

Objective~1 implies that the nonholonomic robot is driven to the target position, and in the meantime, its velocity direction points along the desired direction. Due to the nonholonomic constraints, the robot's velocity direction is its heading direction, and the second limit in Objective~1 can be rewritten as $\lim_{t\to\infty}\|\bm{R}_i\bm{e}_x-\bm{e}_{di}\|=0$, where $\bm{e}_x=[1\ \ 0\ \ 0]^{\rm T}$. Objective~2 guarantees that the nonholonomic robot does not collide with the obstacles in the environment. Objective~3 ensures that there does not exist any collisions among the nonholonomic robots in the swarm.

\section{Motion Planning in Obstacle-Free Environments}\label{sec_VF}

This section considers the motion planning of a single nonholonomic robot in an obstacle-free environment. We firstly design the VF which guides the motion of the nonholonomic robot to the target position and desired heading direction, and then derive the controller based on the proposed VF.

\subsection{Vector Field Design}

Before designing the VF, we provide the notation of coordinate frames. The 3D earth-fixed frame $\bm{\mathcal{F}}_{\rm e}$ is commonly described by the orthogonal coordinates $(x,y,z)$. Actually, $\bm{\mathcal{F}}_{\rm e}$ can also be formulated under the cylindrical coordinates or spherical coordinates. In the following, we let $\bm{\mathcal{F}}_{\alpha\beta\gamma}$ represent the earth-fixed frame given by the coordinates $(\alpha,\beta,\gamma)$. For example, $\bm{\mathcal{F}}_{xyz}$ denotes the 3D orthogonal coordinate frame. In advance of presenting the theorem, we give the following definition about a VF lying in a subspace.

\begin{definition} \label{def1}
    A vector field $\bm{F} : \mathbb{R}^3 \to \mathbb{R}^3$ lies in a nonempty subspace $\bm{\Sigma} \subseteq \mathbb{R}^3$ if there holds $\bm{F}(p) \in \bm{\Sigma}$ for each point $p \in \bm{\Sigma}$. Namely, $F(\bm{\Sigma}) \subseteq \bm{\Sigma}$.
\end{definition}

It is obvious from Definition~\ref{def1} that if a VF $\bm{F}$ lies in a nonempty subspace $\bm{\Sigma}$, then every complete trajectory of $\dot{\bm{p}}(t) = \bm{F}(\bm{p}(t))$ with the initial condition in $\bm{\Sigma}$ will stay in $\bm{\Sigma}$ for $t \ge 0$. In other words, the integral curve of $\bm{F}$ which has nonempty intersection with $\bm{\Sigma}$ is contained in $\bm{\Sigma}$.

The following theorem presents a VF which guides the nonholonomic robot to the target position with the specified heading direction.

\begin{theorem}\label{theo_VF_org}
  Let $\bm{e}_x,\bm{e}_y,\bm{e}_z$ denote the basis of the coordinate frame $\bm{\mathcal{F}}_{xyz}$. Define the navigation vector field (N-VF) $\bm{F}: \mathbb{R}^3 \to \mathbb{R}^3$ by
  \begin{equation}\label{eq_VF_xyz}
    \bm{F}=F_x\bm{e}_x+F_y\bm{e}_y+F_z\bm{e}_z,
  \end{equation}
  where the components are
  \begin{subequations}\label{eq_VF_xyz_compo}
    \begin{align}
      &F_x=x^2-y^2-z^2, \\
      &F_y=2xy, \\
      &F_z=2xz.
    \end{align}
  \end{subequations}
  Then, the following properties hold.
  \begin{enumerate}
    \item Given $a,b\in\mathbb{R}$, where $a^2+b^2\ne0$, $\bm{F}$ lies in the plane $\bm{\Sigma_{ab}}$ defined by $\bm{\Sigma_{ab}}:=\{(x,y,z) \in \mathbb{R}^3 : ay+bz=0\}$. Therefore, $\bm{\Sigma_{ab}}$ is positively invariant w.r.t the dynamics $\dot{\bm{p}}=\bm{F}(\bm{p})$. Furthermore, the $x$-axis is positively invariant.
    \item Every integral curve of $\bm{F}$ starting from the positive $x$-semiaxis will escape to infinity in a finite time. Every integral curve of $\bm{F}$ starting from all other initial conditions in $\mathbb{R}^3$ (except for the trivial case of the origin)  will pass through the origin of $\bm{\mathcal{F}}_{xyz}$, and its tangent vector at the origin points along the positive direction of the $x$-axis.
    \item There is only one singular point of $\bm{F}$, that is, the origin, which is almost globally attractive w.r.t the dynamics $\dot{\bm{p}}=\bm{F}(\bm{p})$.
  \end{enumerate}
\end{theorem}

\begin{figure}[t]
  \centering
  \includegraphics[width=0.25\textwidth]{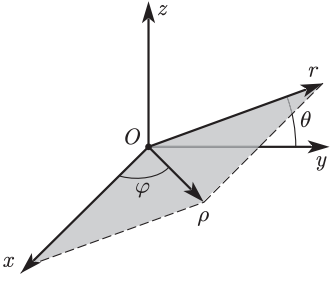}
  \caption{Illustrations of coordinate frames $\bm{\mathcal{F}}_{xyz},\bm{\mathcal{F}}_{xr\theta},\bm{\mathcal{F}}_{\rho\varphi\theta}$. }
  \label{fig_coordinates}
\end{figure}

\begin{proof}
  1) Consider a new coordinate frame $\bm{\mathcal{F}}_{xr\theta}$, and the coordinate transformation between $\bm{\mathcal{F}}_{xyz}$ and $\bm{\mathcal{F}}_{xr\theta}$ is defined by
  \begin{equation}\label{eq_transform_xyz_to_xrtheta}
    x=x,\ y=r\cos\theta,\ z=r\sin\theta.
  \end{equation}
  The illustration of $\bm{\mathcal{F}}_{xr\theta}$ is given in Figure~\ref{fig_coordinates}. Thus, the N-VF components in (\ref{eq_VF_xyz_compo}) can be written as
  \begin{equation}\label{eq_Fxyz_in_xrtheta}
    F_x=x^2-r^2,\ F_y=2xr\cos\theta,\ F_z=2xz\sin\theta.
  \end{equation}
  Let $\bm{e}_x,\bm{e}_r,\bm{e}_{\theta}$ denote the bases of the coordinate frame $\bm{\mathcal{F}}_{xr\theta}$, and then the transformation between $\{\bm{e}_x,\bm{e}_y,\bm{e}_z\}$ and $\{\bm{e}_x,\bm{e}_r,\bm{e}_{\theta}\}$ can be given by
  \begin{equation}\label{eq_basis_transform_xyz_to_xrtheta}
    \begin{bmatrix}
      \bm{e}_x \\
      \bm{e}_y \\
      \bm{e}_z
    \end{bmatrix}=
    \begin{bmatrix}
      1 & 0 & 0 \\
      0 & \cos\theta & -\sin\theta \\
      0 & \sin\theta & \cos\theta
    \end{bmatrix}
    \begin{bmatrix}
      \bm{e}_x \\
      \bm{e}_r \\
      \bm{e}_{\theta}
    \end{bmatrix}.
  \end{equation}
  Please refer to \cite{Surana2014Appendix} for more information about the transformations of coordinate frames.  Substituting (\ref{eq_Fxyz_in_xrtheta}) and (\ref{eq_basis_transform_xyz_to_xrtheta}) into (\ref{eq_VF_xyz}), and after computation, the N-VF $\bm{F}$ can be expressed in the coordinate frame $\bm{\mathcal{F}}_{xr\theta}$ as
  \begin{equation}\label{eq_VF_xrtheta}
    \bm{F}=F_x\bm{e}_x+F_r\bm{e}_r+F_{\theta}\bm{e}_{\theta},
  \end{equation}
  where the components are
  \begin{equation}\label{eq_VF_xrtheta_compo}
    F_x=x^2-r^2,\ F_r=2xr,\ F_{\theta}=0.
  \end{equation}
  Regarding the plane $\bm{\Sigma_{ab}}:ay+bz=0$, it can be reformulated in the coordinate frame $\bm{\mathcal{F}}_{xr\theta}$ as $\bm{\Sigma_{ab}}:r(a\cos\theta+b\sin\theta)=0$.
  Due to $r\ne 0$, the plane $\bm{\Sigma_{ab}}$ can be further simplified to be
  \begin{equation}\label{eq_plane_theta}
    \bm{\Sigma_{ab}}:\theta-\theta_0=0,
  \end{equation}
  where $\theta_0$ is a constant defined by $\theta_0={\rm atan2}(-a,b)$. Thus, the normal vector of the plane $\bm{\Sigma_{ab}}$ in the coordinate frame $\bm{\mathcal{F}}_{xr\theta}$ can be given by $\bm{n}_{\bm{\Sigma_{ab}}}=\bm{e}_{\theta}$. We compute the dot product of the normal vector $\bm{n}_{\bm{\Sigma_{ab}}}$ and the N-VF $\bm{F}$ provided in (\ref{eq_VF_xrtheta}), and it is obtained that
  \begin{equation}
    \bm{n}_{\bm{\Sigma_{ab}}}\cdot\bm{F}=0,
  \end{equation}
  which indicates that $\bm{F}$ is perpendicular to the normal vector of the plane $\bm{\Sigma_{ab}}$. That is to say, $\bm{F}$ always lies in the plane $\bm{\Sigma_{ab}}$ by Definition \ref{def1}. Therefore, any trajectory of $\bm{\dot{p}}=\bm{F}(\bm{p})$ with the initial condition in the plane $\bm{\Sigma_{ab}}$ will stay in $\bm{\Sigma_{ab}}$ for $t \ge 0$. Namely, the plane $\bm{\Sigma_{ab}}$ is positively invariant w.r.t. the dynamics $\bm{\dot{p}}=\bm{F}(\bm{p})$. Since the $x$-axis is the intersection of every such positively invariant plane $\bm{\Sigma_{ab}}$, the $x$-axis is itself positively invariant.

  2) Having proved that the N-VF $\bm{F}$ lies in the plane $\bm{\Sigma_{ab}}$, we will then calculate the integral curve of $\bm{F}$ in $\bm{\Sigma_{ab}}$. For simplicity, another coordinate transformation is made for $\bm{F}$ on the basis of $\bm{\mathcal{F}}_{xr\theta}$. Define a new coordinate frame $\bm{\mathcal{F}}_{\rho\varphi\theta}$, and the transformation between $\bm{\mathcal{F}}_{xr\theta}$ and $\bm{\mathcal{F}}_{\rho\varphi\theta}$ is
  \begin{equation}\label{eq_transform_xrtheta_to_rhophitheta}
    x=\rho\cos\varphi,\ r=\rho\sin\varphi,\ \theta=\theta.
  \end{equation}
  The illustration of $\bm{\mathcal{F}}_{\rho\varphi\theta}$ is also given in Figure~\ref{fig_coordinates}. Then, the N-VF components in (\ref{eq_VF_xrtheta_compo}) can be expressed as
  \begin{equation}\label{eq_Fxrtheta_in_rhophitheta}
    F_x=\rho^2(\cos^2\varphi-\sin^2\varphi),\ F_r=2\rho^2\cos\varphi\sin\varphi,\ F_{\theta}=0.
  \end{equation}
  The basis of the coordinate frame $\bm{\mathcal{F}}_{\rho\varphi\theta}$ is denoted by $\bm{e}_{\rho},\bm{e}_{\varphi},\bm{e}_{\theta}$, and the transformation between $\{\bm{e}_x,\bm{e}_r,\bm{e}_{\theta}\}$ and $\{\bm{e}_{\rho},\bm{e}_{\varphi},\bm{e}_{\theta}\}$ is
  \begin{equation}\label{eq_basis_transform_xrtheta_to_rhophitheta}
    \begin{bmatrix}
      \bm{e}_x \\
      \bm{e}_r \\
      \bm{e}_{\theta}
    \end{bmatrix}=
    \begin{bmatrix}
      \cos\varphi & -\sin\varphi & 0 \\
      \sin\varphi & \cos\varphi & 0 \\
      0 & 0 & 1
    \end{bmatrix}
    \begin{bmatrix}
      \bm{e}_{\rho} \\
      \bm{e}_{\varphi} \\
      \bm{e}_{\theta}
    \end{bmatrix}.
  \end{equation}
  By substituting (\ref{eq_Fxrtheta_in_rhophitheta}) and (\ref{eq_basis_transform_xrtheta_to_rhophitheta}) into (\ref{eq_VF_xrtheta}), the N-VF $\bm{F}$ is expressed in $\bm{\mathcal{F}}_{xr\theta}$ as
  \begin{equation}\label{eq_VF_rhophitheta}
    \bm{F}=F_{\rho}\bm{e}_{\rho}+F_{\varphi}\bm{e}_{\varphi}+F_{\theta}\bm{e}_{\theta},
  \end{equation}
  where the components are
  \begin{equation}\label{eq_VF_rhophitheta_compo}
    F_{\rho}=\rho^2\cos\varphi,\ F_{\varphi}=\rho^2\sin\varphi,\ F_{\theta}=0.
  \end{equation}
  Owing to $F_{\theta}=0$, we can compute the integral curve of $\bm{F}$ in the plane $\bm{\Sigma_{ab}}:\theta-\theta_0=0$ based on the VF components $F_{\rho}$ and $F_{\varphi}$. In $\bm{\Sigma_{ab}}$, the integral curve of $\bm{F}$ is the solution to the following ordinary differential equations
  \begin{equation} \label{eq_ode}
    \frac{{\rm d}\rho}{{\rm d}t}=F_{\rho},\ \rho\frac{{\rm d}\varphi}{{\rm d}t}=F_{\varphi}.
  \end{equation}
  If $\rho=0$ (corresponding to the origin), then $F_\rho=F_\varphi=F_\theta=0$ and hence the origin is an equilibrium point. Now suppose $\rho \ne 0$ for the following discussion.

  If $\varphi = 0$ (corresponding to any point on the \emph{positive} $x$-semiaxis), then $F_\rho=\rho^2$, $F_\varphi=F_\theta=0$. In this case, $\rho=x$. Therefore, by \eqref{eq_ode}, one can directly obtain the analytic expression of any trajectory with the initial condition $(x_0, 0,0)$, where $x_0 := x(0)>0$, on the \emph{positive} $x$-semiaxis as below
  \begin{equation} \label{eq_sol_fet}
    x(t) = \frac{-1}{t - \frac{1}{x_0}}
  \end{equation}
  and $y(t)=z(t)=0$. Note that there is a finite escape time at $t=1/x_0$. Therefore, any trajectory starting from the \emph{positive} $x$-semiaxis will escape to infinity in finite time.

  If $\varphi = \pi$ (corresponding to any point on the \emph{negative} $x$-semiaxis), then $F_\rho=-\rho^2$, $F_\varphi=F_\theta=0$, and $\rho=-x$. It turns out that the the analytic expression of any trajectory with the initial condition $(x_0, 0,0)$, where $x_0 := x(0)<0$, on the \emph{negative} $x$-semiaxis, is still \eqref{eq_sol_fet} and $y(t)=z(t)=0$. However, in this case, there is no finite escape time since $x_0<0$, and instead, the trajectory will converge asymptotically to the origin as $t \to \infty$.

  If $\varphi \ne 0$ and $\varphi \ne \pi$, then it can be derived that
  \begin{equation}
    \frac{{\rm d}\rho}{\rho{\rm d}\varphi}=\frac{F_{\rho}}{F_{\varphi}}=\frac{\cos\varphi}{\sin\varphi},
  \end{equation}
  where (\ref{eq_VF_rhophitheta_compo}) is utilized, and it follows that
  \begin{equation}\label{eq_ODE_curve}
    \frac{1}{\rho}{\rm d}\rho=\cot\varphi{\rm d}\varphi.
  \end{equation}
  By integrating (\ref{eq_ODE_curve}), we can obtain as below
  \begin{equation}\label{eq_inte_curve_rhophitheta}
    \rho=C\sin\varphi,
  \end{equation}
  where $C\in\mathbb{R}\backslash\{0\}$ is a constant. It should be mentioned that (\ref{eq_inte_curve_rhophitheta}) represents a surface rather than a curve in the coordinate frame $\bm{\mathcal{F}}_{\rho\varphi\theta}$. Let $\bm{\Lambda}$ denote the surface given in (\ref{eq_inte_curve_rhophitheta}), and by coordinate transformations (\ref{eq_transform_xyz_to_xrtheta}) and (\ref{eq_transform_xrtheta_to_rhophitheta}), it can be expressed in the coordinate frame $\bm{\mathcal{F}}_{xyz}$ as
  \begin{equation}
    \bm{\Lambda}:x^2+y^2+z^2-C\sqrt{y^2+z^2}=0.
  \end{equation}
  Since it has been proved that the N-VF $\bm{F}$ lies in the plane $\bm{\Sigma_{ab}}$, then the integral curve of $\bm{F}$, denoted by $\bm{\xi}$, can be given by
  \begin{equation}\label{eq_inte_curve}
    \bm{\xi}:
    \left\{
    \begin{aligned}
      & x^2+y^2+z^2-C\sqrt{y^2+z^2}=0\\
      & -\sin\theta_0y+\cos\theta_0z=0
    \end{aligned},
    \right.
  \end{equation}
  where $\theta_0$ is the constant defined in (\ref{eq_plane_theta}). By substituting $(x,y,z)=(0,0,0)$ into (\ref{eq_inte_curve}), it is easily verified that the integral curve $\xi$ passes through the origin.

  Furthermore, we can compute the tangent vector of $\bm{\xi}$ based on the formulation in (\ref{eq_inte_curve}). Since $\bm{\xi}$ is defined as the intersection of two surfaces, the tangent vector can be given by the cross product of the normal vectors of the two surfaces. Define $f_1=x^2+y^2+z^2-C\sqrt{y^2+z^2}$ and $f_2=-\sin\theta_0y+\cos\theta_0z$, and then the normal vectors of $\bm{\Lambda}$ and $\bm{\Sigma_{ab}}$ are given by
  \begin{align*}
    \bm{n}_{\bm{\Lambda}} & =\frac{\partial f_1}{\partial x}\bm{e}_x + \frac{\partial f_1}{\partial y}\bm{e}_y + \frac{\partial f_1}{\partial z}\bm{e}_z \nonumber \\
    & =2x\bm{e}_x+\left(2y-\frac{Cy}{\sqrt{y^2+z^2}}\right)\bm{e}_y+ \left(2z-\frac{Cz}{\sqrt{y^2+z^2}}\right)\bm{e}_z, \\
    \bm{n}_{\bm{\Sigma_{ab}}} & =\frac{\partial f_2}{\partial x}\bm{e}_x + \frac{\partial f_2}{\partial y}\bm{e}_y + \frac{\partial f_2}{\partial z}\bm{e}_z \nonumber \\
    & =-\sin\theta_0\bm{e}_y+\cos\theta_0\bm{e}_z \nonumber \\
    & =-\frac{z}{\sqrt{y^2+z^2}}\bm{e}_y+\frac{y}{\sqrt{y^2+z^2}}\bm{e}_z.
  \end{align*}
  Having obtained $\bm{n}_{\bm{\Lambda}}$ and $\bm{n}_{\bm{\Sigma_{ab}}}$, we can calculate the tangent vector of the integral curve $\bm{\xi}$ as
  \begin{align*}
    \bm{\tau}_{\xi} &= \bm{n}_{\bm{\Sigma}}\times\bm{n}_{\bm{\Lambda}} \nonumber \\
    &= \left(C-2\sqrt{y^2+z^2}\right)\bm{e}_x+\frac{2xy}{\sqrt{y^2+z^2}}\bm{e}_y+\frac{2xz}{\sqrt{y^2+z^2}}\bm{e}_z
  \end{align*}
  One can calculate that
  \begin{equation*}
    \lim_{x,y,z\to 0}\frac{2xy}{\sqrt{y^2+z^2}}=0,\ \lim_{x,y,z\to 0}\frac{2xz}{\sqrt{y^2+z^2}}=0.
  \end{equation*}
  Therefore, the tangent vector $\bm{\tau}_{\xi}$ of the integral curve $\bm{\xi}$ at the origin of the coordinate frame $\bm{\mathcal{F}}_{xyz}$ is
  \begin{equation}
    \bm{\tau}_{\xi}|_{(x,y,z)=(0,0,0)}=C\bm{e}_x,
  \end{equation}
  which indicates that the tangent vector $\bm{\tau}_{\xi}$ at the origin of $\bm{\mathcal{F}}_{xyz}$ points along the $x$-axis. To further prove $\bm{\tau}_{\xi}$ points along the positive direction of the $x$-axis, we define $\bm{\Sigma_\epsilon}=\{(x,y,z)\in\mathbb{R}^3:|x|<\epsilon,y=z=0\}$, where $\epsilon>0$ is sufficiently small. Then, $\bm{\Sigma_\epsilon}$ denotes a small neighborhood of the origin. Thus, it is obtained that
  \begin{equation*}
      \bm{F}(\bm{\Sigma_\epsilon})=\epsilon^2\bm{e}_x,
  \end{equation*}
  indicating the tangent vector points along the positive direction of the $x$-axis.

  3) It is evidently observed from (\ref{eq_VF_xyz_compo}) that $\bm{F}$ vanishes only at the origin. Hence, the origin is the only singular point of $\bm{F}$, i.e., the only equilibrium of the dynamics $\dot{\bm{p}}=\bm{F}(\bm{p})$. For convenience, the \emph{positive} $x$-semiaxis is denoted by $\mathcal{X}_+ := \{(x,0,0) \in \mathbb{R}^3 : x >0\}$ and similarly, the \emph{negative} $x$-semiaxis is denoted by $\mathcal{X}_-$. For any trajectory of $\dot{\bm{p}}=\bm{F}(\bm{p})$ starting from the $x$-axis (excluding the trivial case of the origin), the analytic expression of the trajectory has been shown in \eqref{eq_sol_fet}. It has been shown above that any trajectory starting from $\mathcal{X}_-$ will asymptotically converge to the origin, while the one starting from $\mathcal{X}_+$ will escape to infinity in finite time. Now consider any other initial condition denoted by $\bm{\eta}_0 \in \mathbb{R}^3 \setminus (\mathcal{X}_+ \cup \mathcal{X}_- \cup \{0\})$. It is obvious that there exist $a,b$ such that the plane $\bm{\Sigma_{ab}}$ contains $\bm{\eta}_0$ and $\bm{\Sigma_{ab}}$ is unique. Based on the proof in 2), any trajectory starting from $\bm{\eta}_0$ will be contained in $\bm{\Sigma_{ab}}$ for $t \ge 0$, and the corresponding integral curve is described by \eqref{eq_inte_curve}, where $\theta_0$ is determined by $\bm{\Sigma_{ab}}$. Hence, the trajectory will asymptotically converge to the origin as $t \to \infty$. To sum up, any trajectory starting from $\mathbb{R}^3 \setminus (\mathcal{X}_+ \cup \{0\})$ asymptotically converges to the origin. Namely, the origin is almost globally attractive as $\mathcal{X}_+$ is of measure zero.
\end{proof}

\begin{remark}
  The coordinate frames $\bm{\mathcal{F}}_{xyz},\bm{\mathcal{F}}_{xr\theta},\bm{\mathcal{F}}_{\rho\varphi\theta}$ are all earth-fixed frames, and the relations between these three frames are intuitively illustrated in Figure~\ref{fig_coordinates}. The coordinate frame $\bm{\mathcal{F}}_{xr\theta}$ actually defines a plane which passes through the $x$-axis and forms an angle of $\theta$ w.r.t the $y$-axis. Based on $\bm{\mathcal{F}}_{xr\theta}$, the coordinate frame $\bm{\mathcal{F}}_{\rho\varphi\theta}$ further provides the polar coordinates $(\rho,\varphi)$ in the $xOr$ plane.
\end{remark}

\begin{remark}
  It is indicated by Theorem~\ref{theo_VF_org} that the integral curve of the N-VF $\bm{F}$ always lies in the $xOr$ plane, and the integral curve is expressed as (\ref{eq_inte_curve_rhophitheta}) in the polar coordinates $(\rho,\varphi)$. By substituting (\ref{eq_transform_xrtheta_to_rhophitheta}) into (\ref{eq_inte_curve_rhophitheta}), we can obtain the expression of the integral curve in orthogonal coordinates $(x,r)$, that is, $x^2+(r-\frac{C}{2})^2=(\frac{C}{2})^2$. This demonstrates that the integral curves in the $xOr$ plane are circles located at $(x,r)=(0,\frac{C}{2})$ with radius $|\frac{C}{2}|$. Figure~\ref{fig_VF_xor} illustrates the N-VF $\bm{F}$ and its integral curve with $C>0$ in the $xOr$ plane. Once rotating the $xOr$ plane around the $x$-axis, which is equivalent to choosing different values of $\theta$, we can obtain the integral curves of $\bm{F}$ in the overall 3D space, as shown in Figure~\ref{fig_VF_stream}.
\end{remark}

\begin{figure}[htbp]
  \centering
  \subfigure[Vector field in $xOr$ plane.]{
    \label{fig_VF_xor}
    \includegraphics[width=0.21\textwidth,trim=69 0 70 0,clip]{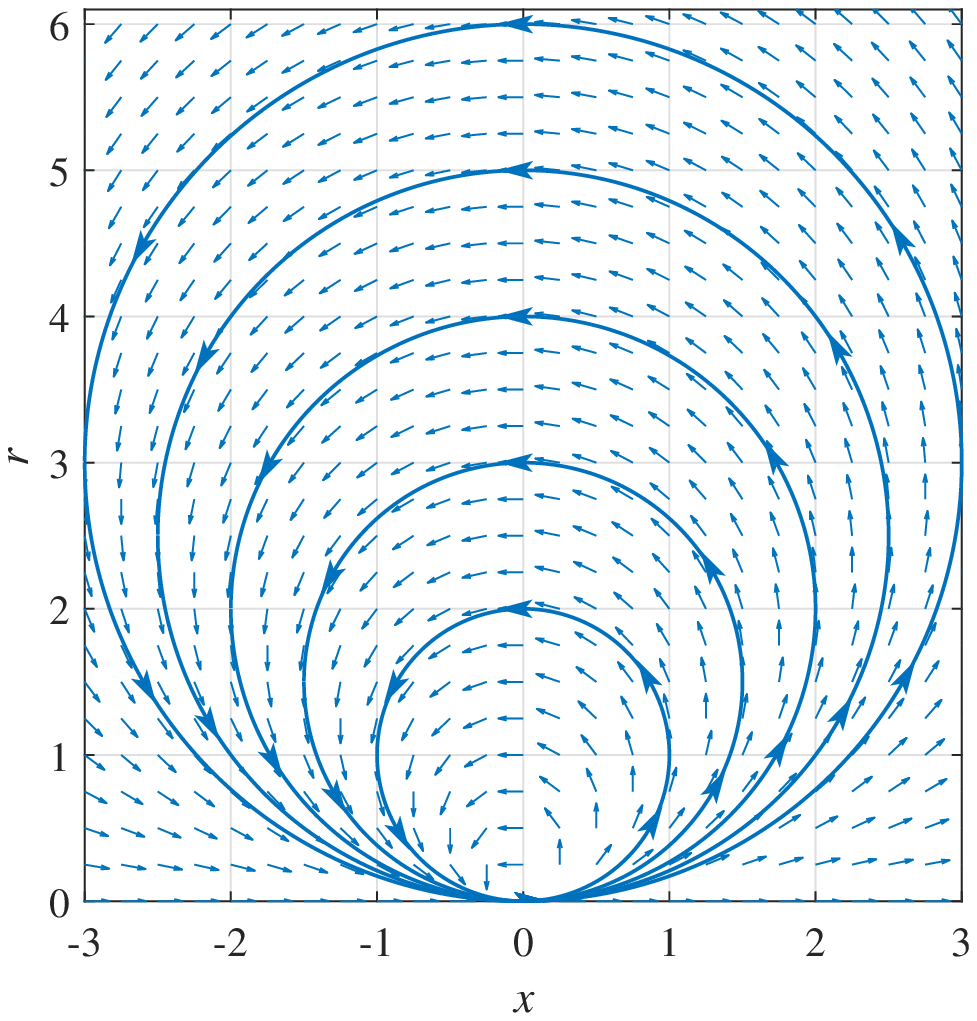}}
  \subfigure[Integral curves of vector field.]{
    \label{fig_VF_stream}
    \includegraphics[width=0.25\textwidth,trim=65 13 60 20,clip]{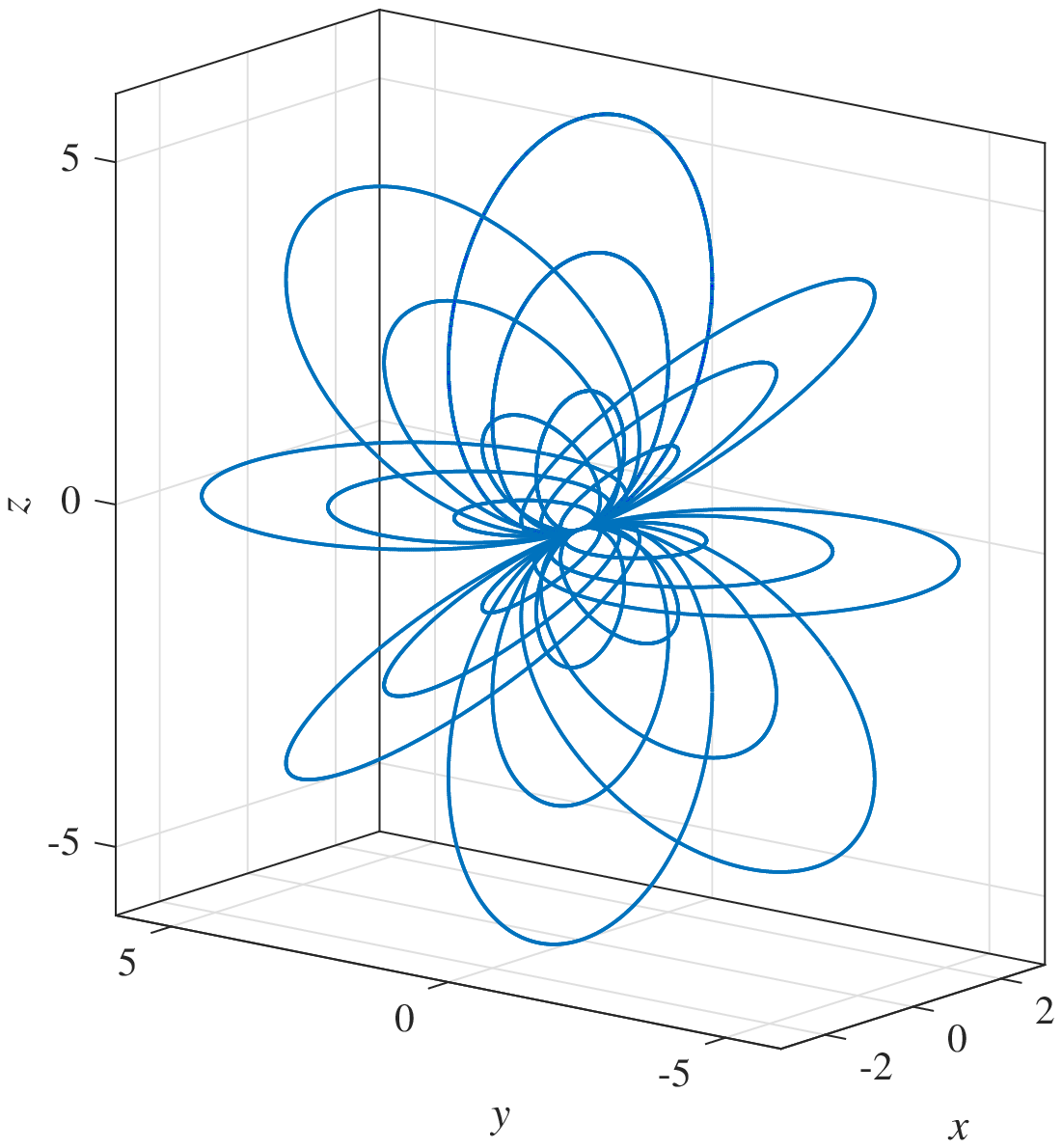}}
  \caption{Navigation vector field $\bm{F}$ and its integral curves.}
  \label{fig_VF_and_stream}
\end{figure}

Theorem~\ref{theo_VF_org} can be extended to an arbitrary goal position and a goal orientation, which is given in the following corollary.

\begin{corollary}\label{cor_VF_arbitra}
  Let $\bm{p}_d=[x_d\ y_d\ z_d]^{\rm T}\in\mathbb{R}^3$ denote a desired position and $\bm{e}_d=\bm{R}_d\bm{e}_x$ a desired direction, where $\bm{R}_d\in{\rm SO(3)}$ represents the rotation relative to the coordinate frame $\bm{\mathcal{F}}_{xyz}$. Define the following N-VF
  \begin{equation}\label{eq_VF_xyz_tilde}
    \tilde{\bm{F}}=\tilde{F}_x\bm{R}_d\bm{e}_x+\tilde{F}_y\bm{R}_d\bm{e}_y+\tilde{F}_z\bm{R}_d\bm{e}_z,
  \end{equation}
  where the components are
  \begin{subequations}\label{eq_VF_xyz_compo_tilde}
    \begin{align}
      &\tilde{F}_x=(x-x_d)^2-(y-y_d)^2-(z-z_d)^2, \\
      &\tilde{F}_y=2(x-x_d)(y-y_d), \\
      &\tilde{F}_z=2(x-x_d)(z-z_d).
    \end{align}
  \end{subequations}
  Then, every integral curve of the N-VF $\tilde{\bm{F}}$ passes through the position $\bm{p}_d$ and points along the direction of $\bm{e}_d$ simultaneously.
\end{corollary}

The N-VF $\tilde{\bm{F}}$ in (\ref{eq_VF_xyz_tilde}) is obviously obtained by a coordinate transformation from the N-VF $\bm{F}$ in (\ref{eq_VF_xyz}). Therefore, without loss of generality, we assume $\bm{p}_d=[0\ \ 0\ \ 0]^{\rm T}$ and $\bm{e}_d = \bm{e}_x = [1\ \ 0\ \ 0]^{\rm T}$ henceforth.

\subsection{Controller Design}

Theorem~\ref{theo_VF_org} provides a VF which passes through a point with a desired direction. This naturally inspires us that we can make the nonholonomic robot to follow the direction of the VF, so that it will be able to arrive at the goal point with a specified heading direction. Note that the nonholonomic robot is a 3D rigid body, and its heading direction is actually the $x$-axis direction of the body-fixed frame $\bm{\mathcal{F}}_{\rm b}$. Although the VF in Theorem~\ref{theo_VF_org} provides the $x$-axis direction of $\bm{\mathcal{F}}_{\rm b}$, the directions of $y$-axis and $z$-axis are undetermined yet, leading to the result that the angular velocity $\bm{\Omega}^{\wedge}$ of the nonholonomic robot cannot be derived either.

Hence, in the following, we design two extra VFs which represents the $y$-axis and $z$-axis directions of the body-fixed frame $\bm{\mathcal{F}}_{\rm b}$, and then plus the N-VF $\bm{F}$ in Theorem~\ref{theo_VF_org}, an auxiliary attitude matrix $\bm{R}_a$ can be constructed based on these three VFs. By tracking the attitude $\bm{R}_a$ with an angular velocity controller, the nonholonomic robot can align its heading direction with the N-VF $\bm{F}$, so as to accomplish the task of motion planning.

\begin{lemma}\label{lem_VF_G}
  Define a VF $\bm{G}$ as follows
  \begin{equation}\label{eq_VF_G_xyz}
    \bm{G}=G_x\bm{e}_x+G_y\bm{e}_y+G_z\bm{e}_z,
  \end{equation}
  where the components are
  \begin{subequations}\label{eq_VF_G_xyz_comp}
    \begin{align}
      &G_x=2x(y^2+z^2), \\
      &G_y=y\left(y^2+z^2-x^2 \right), \\
      &G_z=z\left(y^2+z^2-x^2\right).
    \end{align}
  \end{subequations}
  Then, $\bm{G}$ is in the $xOr$ plane and orthogonal to the N-VF $\bm{F}$.
\end{lemma}

\begin{proof}
  By substituting the coordinate transformation (\ref{eq_transform_xyz_to_xrtheta}) and basis transformation (\ref{eq_basis_transform_xyz_to_xrtheta}) into (\ref{eq_VF_G_xyz}) and (\ref{eq_VF_G_xyz_comp}), $\bm{G}$ can be rewritten in the frame $\bm{\mathcal{F}}_{xr\theta}$ as
  \begin{equation}\label{eq_VF_G_xrtheta}
    \bm{G}=G_x\bm{e}_x+G_r\bm{e}_r+G_{\theta}\bm{e}_{\theta},
  \end{equation}
  where the components are
  \begin{equation}\label{eq_VF_G_xrtheta_compo}
    G_x=2xr^2,\ G_r=r(r^2-x^2),\ G_{\theta}=0.
  \end{equation}
  As illustrated in the proof of Theorem~\ref{theo_VF_org}, $G_{\theta}=0$ demonstrates that $\bm{G}$ lies in the $xOr$ plane. To prove the orthogonality, we compute the dot product of $\bm{F}$ and $\bm{G}$, which is given by
  \begin{equation}\label{eq_dot_product_F_G}
    \bm{F}\cdot\bm{G}=F_xG_x+F_rG_r+F_{\theta}G_{\theta}
  \end{equation}
  Substituting (\ref{eq_VF_xrtheta_compo}) and (\ref{eq_VF_G_xrtheta_compo}) into (\ref{eq_dot_product_F_G}),   we have
  \begin{equation}\label{eq_F_cdot_G}
    \bm{F}\cdot\bm{G}=0,
  \end{equation}
  which implies the VFs $\bm{F}$ and $\bm{G}$ are orthogonal.
\end{proof}

Based on $\bm{F}$ and $\bm{G}$, we define another VF $\bm{H}$ by the cross product of $\bm{F}$ and $\bm{G}$, that is
\begin{equation}\label{eq_VF_H_xyz}
  \bm{H} =\bm{G}\times\bm{F} =H_x\bm{e}_x+H_y\bm{e}_y+H_z\bm{e}_z,
\end{equation}
where
\begin{subequations}\label{eq_VF_H_xyz_comp}
    \begin{align}
      &H_x=0, \\
      &H_y=-z\left(x^2+y^2+z^2 \right), \\
      &H_z=y\left(x^2+y^2+z^2\right).
    \end{align}
\end{subequations}
Therefore, it follows from (\ref{eq_F_cdot_G}) and (\ref{eq_VF_H_xyz}) that each two VFs of $\bm{F},\bm{G},\bm{H}$ are orthogonal, which indeed defines a 3D Cartesian coordinate frame at every point $(x,y,z)$ in $\mathbb{R}^3$. Then, an attitude matrix $\bm{R}_a$ can be constructed to describe the attitude of such a coordinate frame. For simplicity, let vectors $\bm{\zeta}_X,\bm{\zeta}_Y,\bm{\zeta}_Z$ denote the components of $\bm{F},\bm{H},\bm{G}$ at each point, and then it follows that
\begin{equation}\label{eq_zeta_XYZ}
  \bm{\zeta}_X=
  \begin{bmatrix}
    F_x \\ F_y \\ F_z
  \end{bmatrix},\quad
  \bm{\zeta}_Y=
  \begin{bmatrix}
    H_x \\ H_y \\ H_z
  \end{bmatrix},\quad
  \bm{\zeta}_Z=
  \begin{bmatrix}
    G_x \\ G_y \\ G_z
  \end{bmatrix},
\end{equation}
where the components are given in (\ref{eq_VF_xyz_compo})(\ref{eq_VF_H_xyz_comp})(\ref{eq_VF_G_xyz_comp}), respectively. Note that the vectors $\bm{\zeta}_X,\bm{\zeta}_Y,\bm{\zeta}_Z$ are orthogonal to each other, based on which we can define the following auxiliary attitude matrix
\begin{equation}\label{eq_auxi_atti_matrix}
  \bm{R}_a=
  \begin{bmatrix}
    \frac{\bm{\zeta}_X}{\|\bm{\zeta}_X\|} & \frac{\bm{\zeta}_Y}{\|\bm{\zeta}_Y\|} & \frac{\bm{\zeta}_Z}{\|\bm{\zeta}_Z\|}
  \end{bmatrix}\in{\rm SO(3)}.
\end{equation}
It can be observed from (\ref{eq_VF_xyz_compo}) that the N-VF $\bm{F}$ is singular, that is, $F_x=F_y=F_z=0$, only at the origin (the goal point). Thus, the auxiliary attitude matrix $\bm{R}_a$ is well defined almost everywhere except for the destination. The time derivative of $\bm{R}_a$ is
\begin{equation}\label{eq_dot_auxi_atti_matrix}
  \dot{\bm{R}}_a=
  \begin{bmatrix}
    \frac{\rm d}{{\rm d}t}\left(\frac{\bm{\zeta}_X}{\|\bm{\zeta}_X\|}\right) & \frac{\rm d}{{\rm d}t}\left(\frac{\bm{\zeta}_Y}{\|\bm{\zeta}_Y\|}\right) & \frac{\rm d}{{\rm d}t}\left(\frac{\bm{\zeta}_Z}{\|\bm{\zeta}_Z\|}\right)
  \end{bmatrix},
\end{equation}
where
\begin{equation}\label{eq_dot_auxi_unit_vec}
  \frac{\rm d}{{\rm d}t}\left(\frac{\bm{\zeta}_*}{\|\bm{\zeta}_*\|}\right)=\frac{\dot{\bm{\zeta}}_*}{\|\bm{\zeta}_*\|}-(\bm{\zeta}_*^{\rm T}\dot{\bm{\zeta}_*})\frac{\bm{\zeta}_*}{\|\bm{\zeta}_*\|^3}.
\end{equation}

\begin{lemma}\label{lem_kine_R_a}
    Define $\bm{\Omega}_a^\wedge:=\bm{R}_a^{\rm T}\dot{\bm{R}}_a$, and then there holds $\bm{\Omega}_a^\wedge\in\mathfrak{so}(3)$.
\end{lemma}

\begin{proof}
    Please refer to \ref{app_A}.
\end{proof}

\begin{lemma}\label{lem_align_to_track}
  Given the N-VF $\bm{F}$ in \eqref{eq_VF_xyz} and the auxiliary attitude $\bm{R}_a$ in \eqref{eq_auxi_atti_matrix}, there holds $\left\|\bm{R}\frac{\bm{v}}{\|\bm{v}\|}-\frac{\bm{\zeta}_X}{\|\bm{\zeta}_X\|}\right\|=0$ if $\bm{R}_a^{\rm T}\bm{R}=\bm{I}$, where $\bm{R}$ and $\bm{v}$ are the attitude and linear velocity of the nonholonomic robot, respectively.
\end{lemma}

\begin{proof}
  The equality $\bm{R}_a^{\rm T}\bm{R}=\bm{I}$ implies $\bm{R}=\bm{R}_a$. Then, it follows that $\bm{R}\bm{e}_x=\bm{R}_a\bm{e}_x$, where $\bm{e}_x=[1\ \ 0\ \ 0]^{\rm T}$. Considering the definition of $\bm{R}_a$ in \eqref{eq_auxi_atti_matrix}, we have $\bm{R}_a\bm{e}_x=\frac{\bm{\zeta}_X}{\|\bm{\zeta}_X\|}$. Hence, $\bm{R}\bm{e}_x=\frac{\bm{\zeta}_X}{\|\bm{\zeta}_X\|}$, i.e., $\left\|\bm{R}\bm{e}_x-\frac{\bm{\zeta}_X}{\|\bm{\zeta}_X\|}\right\|=0$. Due to the nonholonomic constraints, the robot's linear velocity can be expressed as $\bm{v}=[v_x\ \ 0\ \ 0]^{\rm T}$, followed by $\frac{\bm{v}}{\|\bm{v}\|}=\bm{e}_x$. Thus, there holds $\left\|\bm{R}\frac{\bm{v}}{\|\bm{v}\|}-\frac{\bm{\zeta}_X}{\|\bm{\zeta}_X\|}\right\|=\left\|\bm{R}\bm{e}_x-\frac{\bm{\zeta}_X}{\|\bm{\zeta}_X\|}\right\|=0$.
\end{proof}

Note that the vector $\bm{\zeta}_X$ is indeed the direction of the N-VF $\bm{F}$. Then, Lemma~\ref{lem_align_to_track} demonstrates intuitively that the nonholonomic robot's heading direction is aligned with the N-VF $\bm{F}$ if the nonholonomic robot's attitude $\bm{R}$ tracks the auxiliary attitude $\bm{R}_a$. It is evident from Theorem~\ref{theo_VF_org} that the nonholonomic robot will arrive at the goal point with the desired heading direction if it moves along the N-VF $\bm{F}$, which is further converted by Lemma~\ref{lem_align_to_track} to the attitude tracking of $\bm{R}$ w.r.t $\bm{R}_a$. In other words, once the nonholonomic robot tracks the attitude matrix $\bm{R}$, it can follow the N-VF $\bm{F}$ to reach the goal point and the desired orientation. Therefore, in the following, we will propose the control inputs of motion planning by solving the problem of attitude tracking.

\begin{lemma}\label{lem_R_e_stabilize}
  The nonholonomic robot's attitude $\bm{R}$ exponentially tracks the auxiliary attitude $\bm{R}_a$ under the angular velocity control law given below
  \begin{equation}\label{eq_Omega_design}
    \bm{\Omega}^{\wedge}=-k_w\log_{\rm SO(3)}(\bm{R}_a^{-1}\bm{R})+\bm{R}^{-1}\dot{\bm{R}}_a\bm{R}_a^{-1}\bm{R},
  \end{equation}
  where $k_w>0$ is the scalar control gain, and $\log_{\rm SO(3)}$ is the logarithmic map on the Lie group ${\rm SO(3)}$.
\end{lemma}

\begin{proof}
  Define the attitude tracking error $\bm{R}_e=\bm{R}_a^{-1}\bm{R}$, and the time derivative of $\bm{R}_e$ is
  \begin{align}\label{eq_dot_R_e}
    \dot{\bm{R}}_e &= -\bm{R}_a^{-1}\dot{\bm{R}}_a\bm{R}_a^{-1}\bm{R}+\bm{R}_a^{-1}\dot{\bm{R}} \\
                   &= \bm{R}_e(\bm{\Omega}^{\wedge}-\bm{R}^{-1}\dot{\bm{R}}_a\bm{R}_e),
   \end{align}
  where the rotation kinematics of the nonholonomic robot is employed. Let $\bm{\Omega}_e^{\wedge}$ denote the velocity of the error system, that is
  \begin{equation}\label{eq_Omega_e}
    \bm{\Omega}_e^{\wedge}=\bm{\Omega}^{\wedge}-\bm{R}^{-1}\dot{\bm{R}}_a\bm{R}_e.
  \end{equation}
  According to Lemma~\ref{lem_kine_R_a}, there holds $\dot{\bm{R}}_a=\bm{R}_a\bm{\Omega}_a^\wedge$, where $\bm{\Omega}_a^\wedge\in\mathfrak{so}(3)$. By substituting it into (\ref{eq_Omega_e}), we have
  \begin{align}\label{eq_Omega_e_v2}
    \bm{\Omega}_e^{\wedge} &= \bm{\Omega}^{\wedge}-\bm{R}^{-1}\bm{R}_a\bm{\Omega}_a^\wedge\bm{R}_e \nonumber \\
    &= \bm{\Omega}^{\wedge}-\bm{R}_e^{-1}\bm{\Omega}_a^\wedge\bm{R}_e \nonumber \\
    &= \bm{\Omega}^{\wedge}-{\rm Ad}_{\bm{R}_e^{-1}}\bm{\Omega}_a^\wedge,
  \end{align}
  where the adjoint map ${\rm Ad}_{\bm{R}}:\mathfrak{so}(3)\to\mathfrak{so}(3)$ is defined by ${\rm Ad}_{\bm{R}}\bm{\Omega}^\wedge=\bm{R}\bm{\Omega}^\wedge\bm{R}^{-1}$, $\forall\bm{\Omega}^\wedge\in\mathfrak{so}(3)$. Then, it is obtained that $\bm{\Omega}_e^{\wedge}\in\mathfrak{so}(3)$. Thus, the error dynamics can be described by
  \begin{equation}\label{eq_dot_R_e_simp}
    \dot{\bm{R}_e}=\bm{R}_e\bm{\Omega}_e^{\wedge}.
  \end{equation}
  Based on the definition of $\bm{R}_e$, if there holds $\bm{R}_e=\bm{I}$, where $\bm{I}$ is the identity matrix, then it is obtained that $\bm{R}=\bm{R}_a$, indicating the mission of attitude tracking is achieved. It has been proposed in \cite{Bullo1995Proportional} that regarding the system $\dot{\bm{R}}=\bm{R}\bm{\Omega}^{\wedge}$ in the Lie group ${\rm SO(3)}$, the control law $\bm{\Omega}^{\wedge}=-k_w\log_{\rm SO(3)}(\bm{R})$ exponentially stabilizes the state $\bm{R}$ to $\bm{I}$. Hence, for the purpose of attitude tracking, the error system's velocity $\bm{\Omega}_e^{\wedge}$ can be designed as
  \begin{equation}\label{eq_Omega_e_design}
    \bm{\Omega}_e^{\wedge}=-k_w\log_{\rm SO(3)}(\bm{R}_e),
  \end{equation}
  which is able to realize $\bm{R}_e\to\bm{I}$ exponentially. By substituting (\ref{eq_Omega_e_design}) into (\ref{eq_Omega_e}), we obtain the angular velocity controller of the nonholonomic robot as given in (\ref{eq_Omega_design}).
\end{proof}

Regarding the control law of the linear velocity $v_x$, the distance error to the desired position (i.e., the origin of the earth-fixed frame) is introduced, so that $v_x$ is given by
\begin{equation}\label{eq_vx_design}
  v_x=k_v\|\bm{p}\|.
\end{equation}
To summarize, the control law for nonholonomic robot motion planning is presented in the following theorem.

\begin{theorem}\label{theo_VF_org_control}
  Under the control laws (\ref{eq_Omega_design}) and (\ref{eq_vx_design}), the closed-loop system of the nonholonomic robot \eqref{eq_kine_UVA} satisfies $\lim_{t\to\infty}\|\bm{p}\|=0$ and $\lim_{t\to\infty}\left\|\bm{R}\frac{\bm{v}}{\|\bm{v}\|}-\bm{e}_{x}\right\|=0$ from all initial positions in $\mathbb{R}^3 \setminus \mathcal{X}_{0+}$, where $\mathcal{X}_{0+}:=\{(x,0,0)\in\mathbb{R}^3:x\ge 0\}$.
\end{theorem}

\begin{proof}
  According to Lemma~\ref{lem_R_e_stabilize}, the angular velocity control law \eqref{eq_Omega_design} makes the robot attitude $\bm{R}$ exponentially track the auxiliary attitude $\bm{R}_a$. Then, the kinematics of the robot's translational motion \eqref{eq_kine_trans} can be rewritten as
  \begin{equation}\label{eq_dot_p_v1}
      \dot{\bm{p}}=\bm{R}_a\bm{v}.
  \end{equation}
  By substituting $\bm{v}=k_v\|\bm{p}\|\bm{e}_x$ into \eqref{eq_dot_p_v1}, we have
  \begin{equation}\label{eq_dot_p_v2}
      \dot{\bm{p}}=k_v\|\bm{p}\|\bm{R}_a\bm{e}_x.
  \end{equation}
  Based on the definition of $\bm{R}_a$ in \eqref{eq_auxi_atti_matrix}, there holds $\bm{R}_a\bm{e}_x=\frac{\bm{\zeta}_X}{\|\bm{\zeta}_X\|}$, which is the unit vector along the N-VF $\bm{F}$ according to \eqref{eq_zeta_XYZ}. Then, the dynamics \eqref{eq_dot_p_v2} can be expressed as
  \begin{equation}\label{eq_dot_p_v3}
      \dot{\bm{p}}=k_v\|\bm{p}\|\bm{F}(\bm{p}).
  \end{equation}
  Thus, the integral curve of the dynamics \eqref{eq_dot_p_v3} is equivalent to that of $\dot{\bm{p}}=\bm{F}(\bm{p})$ \cite[Proposition 1.14]{chicone2006ordinary}, which is given in \eqref{eq_inte_curve}, and the state $\bm{p}$ decided by \eqref{eq_dot_p_v3} will evolve on the integral curve \eqref{eq_inte_curve} at the speed of $k_v\|\bm{p}\|$. Then, according to Theorem~\ref{theo_VF_org}, from all initial conditions in $\mathbb{R}^3\backslash\mathcal{X}_{0+}$, the position $\bm{p}$ will converge to the origin, that is, $\lim_{t\to\infty}\|\bm{p}(t)\|=0$. Furthermore, the tangent vector of the integral curve at the origin will point to the positive $x$-axis. Since the direction of the tangent vector is actually the direction of the velocity $\bm{v}$, which is expressed in the earth-fixed frame $\bm{\mathcal{F}}_{xyz}$ as $\bm{R}\frac{\bm{v}}{\|\bm{v}\|}$, then there holds $\lim_{t\to\infty}\left\|\bm{R}\frac{\bm{v}}{\|\bm{v}\|}-\bm{e}_{x}\right\|=0$.
\end{proof}

\begin{remark}
  Theorem~\ref{theo_VF_org_control} provides the control law in the case where the goal point is the origin and goal direction is the $x$-axis. Regarding arbitrary desired point and direction, Corollary~\ref{cor_VF_arbitra} has provided the related N-VF $\tilde{\bm{F}}$. Then, the angular velocity in (\ref{eq_Omega_design}) can be correspondingly revised by reconstructing the auxiliary attitude matrix $\bm{R}_a$ with the aid of $\tilde{\bm{F}}$ in (\ref{eq_VF_xyz_tilde}). Besides, the linear velocity in (\ref{eq_vx_design}) can also be reformulated to be $v_x=k_v\|\bm{p}-\bm{p}_d\|$.
\end{remark}

\begin{remark}
  The formulation in (\ref{eq_VF_G_xyz_comp}) and (\ref{eq_VF_H_xyz_comp}) is merely one possible way to construct two auxiliary VFs $\bm{G}$ and $\bm{H}$, which in fact represent the directions of $z-,y$-axis of the body-fixed frame $\bm{\mathcal{F}}_{\rm b}$. Since there only exists a requirement for the nonholonomic robot's heading direction (i.e., the direction of the $x$-axis), the directions of the $y-,z$-axis can be freely designed as long as they can form a set of basis in the 3D Cartesian frame. Of course, it should be mentioned that the VFs $\bm{G}$ and $\bm{H}$ presented in this paper have explicit geometric meanings. More specifically, $\bm{G}$ represents the outward normal vector of the integral curves in the $xOr$ plane (the integral curves are the circles shown in Figure~\ref{fig_VF_xor}), while $\bm{H}$ represents the normal vector of the $xOr$ plane.
\end{remark}

\section{Motion Planning in Obstacle-Cluttered Environments}\label{sec_OA_VF}

In this section, we consider the motion planning of a nonholonomic robot in an obstacle-cluttered environment. According to Assumption~\ref{assum_dist_ob}, the nonholonomic robot can enter the reactive area of one obstacle at each time instant. Therefore, for the sake of simplicity, we remove the subscript denoting the obstacle's label and consider the problem of circumventing one obstacle. Based on Section~\ref{subsec_pf}, the obstacle avoidance problem is mathematically formulated by $\verb"dist"(\bm{\mathcal{C}}_{o},\bm{\mathcal{A}}_{o})>0$, where $\bm{\mathcal{C}}_{o}$ is the robot's safe area and $\bm{\mathcal{A}}_{o}$ is the obstacle area. Note that the robot's safe area $\bm{\mathcal{C}}_{o}$ is a sphere with radius $r_c$. By adjusting the parameter $\bar{c}$ of the level surface $\bm{\Upsilon}(\bm{p};\bm{p}_{o})=\bar{c}$, the obstacle area $\bm{\mathcal{A}}_{o}$ can be enlarged radially with $r_c$. In this way, the objective of obstacle avoidance can be reformulated by $\verb"dist"(\bm{p},\bm{\mathcal{A}}_{o})>0$, where $\bm{p}$ is the position of the nonholonomic robot. Therefore, in the following, the obstacle avoidance VF (OA-VF) is presented based on the requirement of $\verb"dist"(\bm{p},\bm{\mathcal{A}}_{o})>0$.

To design the OA-VF, we firstly compute the normal vector at each point on the surface $\bm{\Upsilon}(\bm{p};\bm{p}_o)=k$ $(1\le k\le \bar{c})$, which is given by
\begin{equation}\label{n_Upsilon}
  \bm{n}_{\bm{\Upsilon}}=\frac{\partial \bm{\Upsilon}}{\partial x}\bm{e}_x+\frac{\partial \bm{\Upsilon}}{\partial y}\bm{e}_y+\frac{\partial \bm{\Upsilon}}{\partial z}\bm{e}_z.
\end{equation}
The time derivative of $\bm{n}_{\bm{\Upsilon}}$ is
\begin{equation*}
  \dot{\bm{n}}_{\bm{\Upsilon}}=\frac{\partial \bm{n}_{\bm{\Upsilon}}}{\partial \bm{p}}\dot{\bm{p}}+\frac{\partial \bm{n}_{\bm{\Upsilon}}}{\partial \bm{p}_o}\dot{\bm{p}_o}.
\end{equation*}
For the static obstacles, there holds $\dot{\bm{p}_o}=\bm{0}$ and the time derivative of $\bm{n}_{\bm{\Upsilon}}$ degenerates to $\dot{\bm{n}}_{\bm{\Upsilon}}=\frac{\partial \bm{n}_{\bm{\Upsilon}}}{\partial \bm{p}}\dot{\bm{p}}$. Then, based on the normal vector $\bm{n}_{\bm{\Upsilon}}$, we propose the OA-VF according to two different cases, which are classified by whether the normal vector $\bm{n}_{\bm{\Upsilon}}$ is collinear with the N-VF $\bm{F}$ given in (\ref{eq_VF_xyz}). In the following, we assume the target point $\bm{p}_d$ is the origin of the earth-fixed frame $\bm{\mathcal{F}}_{\rm e}$ for simplicity.


\emph{1) $\bm{n}_{\bm{\Upsilon}}$ and $\bm{F}$ are not collinear ($\bm{n}_{\bm{\Upsilon}}\nparallel\bm{F}$)}

Define a new vector $\bm{\tau}_{\bm{\Upsilon}}^a$ by the following cross product
\begin{equation}\label{tau_Upsilon_1}
  \bm{\tau}_{\bm{\Upsilon}}^a=\bm{n}_{\bm{\Upsilon}}\times\bm{F}.
\end{equation}
Substituting (\ref{eq_VF_xyz}) and (\ref{n_Upsilon}) into (\ref{tau_Upsilon_1}), we obtain that
\begin{equation}
  \bm{\tau}_{\bm{\Upsilon}}^a=T^a_x\bm{e}_x+T^a_y\bm{e}_y+T^a_z\bm{e}_z,
\end{equation}
where the components $T^a_x,T^a_y,T^a_z$ are
\begin{align*}
  T^a_x &= \frac{\partial \bm{\Upsilon}}{\partial y}F_z-\frac{\partial \bm{\Upsilon}}{\partial z}F_y, \nonumber \\
  T^a_y &= \frac{\partial \bm{\Upsilon}}{\partial z}F_x-\frac{\partial \bm{\Upsilon}}{\partial x}F_z, \nonumber \\
  T^a_z &= \frac{\partial \bm{\Upsilon}}{\partial x}F_y-\frac{\partial \bm{\Upsilon}}{\partial y}F_z.
\end{align*}
Then, we define another vector $\bm{\tau}_{\bm{\Upsilon}}^b$ as follows
\begin{equation}\label{tau_Upsilon_2_case1}
  \bm{\tau}_{\bm{\Upsilon}}^b=\bm{\tau}_{\bm{\Upsilon}}^a\times\bm{n}_{\bm{\Upsilon}}.
\end{equation}
By substituting (\ref{n_Upsilon}) and (\ref{tau_Upsilon_1}) into (\ref{tau_Upsilon_2_case1}), it is obtained that
\begin{equation}
  \bm{\tau}_{\bm{\Upsilon}}^b=T^b_x\bm{e}_x+T^b_y\bm{e}_y+T^b_z\bm{e}_z,
\end{equation}
where the components $T^b_x,T^b_y,T^b_z$ are
\begin{align*}
  T^b_x &= \left((\frac{\partial \bm{\Upsilon}}{\partial z})^2+(\frac{\partial \bm{\Upsilon}}{\partial y})^2\right)F_x
  -\frac{\partial \bm{\Upsilon}}{\partial x}\left(\frac{\partial \bm{\Upsilon}}{\partial z}F_z+\frac{\partial \bm{\Upsilon}}{\partial y}F_y\right), \nonumber \\
  T^b_y &= \left((\frac{\partial \bm{\Upsilon}}{\partial x})^2+(\frac{\partial \bm{\Upsilon}}{\partial z})^2\right)F_y
  -\frac{\partial \bm{\Upsilon}}{\partial y}\left(\frac{\partial \bm{\Upsilon}}{\partial x}F_x+\frac{\partial \bm{\Upsilon}}{\partial z}F_z\right), \nonumber \\
  T^b_z &= \left((\frac{\partial \bm{\Upsilon}}{\partial y})^2+(\frac{\partial \bm{\Upsilon}}{\partial x})^2\right)F_z
  -\frac{\partial \bm{\Upsilon}}{\partial z}\left(\frac{\partial \bm{\Upsilon}}{\partial y}F_y+\frac{\partial \bm{\Upsilon}}{\partial x}F_x\right).
\end{align*}


\emph{2) $\bm{n}_{\bm{\Upsilon}}$ and $\bm{F}$ are collinear ($\bm{n}_{\bm{\Upsilon}}\parallel\bm{F}$)}

When $\bm{n}_{\bm{\Upsilon}}$ and $\bm{F}$ are lying in the same straight line, the VF $\bm{\tau}_{\bm{\Upsilon}}^a$ is undefined since the cross product in (\ref{tau_Upsilon_1}) is $\bm{0}$. In this case, we directly define
\begin{align}
  \bm{\tau}_{\bm{\Upsilon}}^a&=\bm{H}, \\
  \bm{\tau}_{\bm{\Upsilon}}^b&=\bm{G}, \label{tau_Upsilon_2_case2}
\end{align}
where $\bm{G}$ and $\bm{H}$ are VFs given in (\ref{eq_VF_G_xyz}) and (\ref{eq_VF_H_xyz}).

\begin{figure}[t]
  \centering
  \includegraphics[width=0.28\textwidth,trim=0 0 0 0,clip]{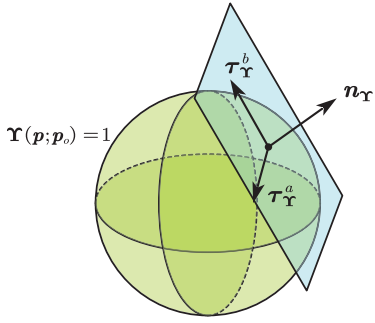}
  \caption{Vector fields on the obstacle surface.}
  \label{fig_obstacle}
\end{figure}

Figure~\ref{fig_obstacle} intuitively depicts the VFs $\bm{n}_{\bm{\Upsilon}},\bm{\tau}_{\bm{\Upsilon}}^a,\bm{\tau}_{\bm{\Upsilon}}^b$ on the obstacle surface $\bm{\Upsilon}(\bm{p};\bm{p}_o)=1$. From the geometrical point of view, $\bm{n}_{\bm{\Upsilon}}$ is the normal vector of the obstacle surface, while $\bm{\tau}_{\bm{\Upsilon}}^a$ and $\bm{\tau}_{\bm{\Upsilon}}^b$ are the tangential vectors lying in the plane which is orthogonal to $\bm{n}_{\bm{\Upsilon}}$. Moreover, $\bm{\tau}_{\bm{\Upsilon}}^a$ and $\bm{\tau}_{\bm{\Upsilon}}^b$ are also orthogonal to each other. Thus, based on the N-VF $\bm{F}$ and the tangential VF $\bm{\tau}_{\bm{\Upsilon}}^b$, we propose the OA-VF in the reactive area $\bm{\mathcal{A}}_r$ as follows
\begin{equation}\label{eq_F_ob}
  \bm{F}_{OA}=\chi\bm{F}+(1-\chi)\bm{\tau}_{\bm{\Upsilon}}^b,
\end{equation}
where $\chi$ is a smooth function defined by
\begin{equation}\label{eq_chi}
  \chi=\left\{
  \begin{aligned}
    0,\quad     \ &\bm{p}\in\{\bm{\Upsilon}(\bm{p};\bm{p}_o)\le 1\} \\
    S(\bm{p}),\ \ &\bm{p}\in\{1<\bm{\Upsilon}(\bm{p};\bm{p}_o)<\bar{c}\} \\
    1,\quad     \ &\bm{p}\in\{\bm{\Upsilon}(\bm{p};\bm{p}_o)\ge \bar{c}\}
  \end{aligned}
  \right.
\end{equation}
and $S(\bm{p})$ is a smooth function valued in $(0,1)$, which can be constructed by the typical bump functions \cite{fry2002smooth}.

The primary requirement for $\bm{F}_{OA}$ is to guarantee the obstacle avoidance. Besides, since $\bm{F}_{OA}$ is a composite VF of $\bm{F}$ and $\bm{\tau}_{\Upsilon}^b$, we should first investigate whether $\bm{F}_{OA}$ has singular points where the nonholonomic robot would get stuck possibly.
\begin{theorem}\label{theo_F_OA}
  The OA-VF $\bm{F}_{OA}$ proposed in (\ref{eq_F_ob}) has the following properties:
  \begin{enumerate}
    \item $\bm{F}_{OA}$ never penetrates the obstacle area $\bm{\mathcal{A}}_o$;
    \item $\bm{F}_{OA}$ does not vanish in the reactive area $\bm{\mathcal{A}}_r$.
  \end{enumerate}
\end{theorem}

\begin{proof}
  1) Impenetrability of the obstacle area $\bm{\mathcal{A}}_o$ can be guaranteed if the OA-VF $\bm{F}_{OA}$ projects zero onto the normal vector $\bm{n}_{\bm{\Upsilon}}$ of the obstacle surface $\bm{\Upsilon}(\bm{p};\bm{p}_o)=1$. In light of this fact, we should compute the dot product of $\bm{F}_{OA}$ and $\bm{n}_{\bm{\Upsilon}}$. According to (\ref{eq_F_ob}), the OA-VF $\bm{F}_{OA}$ on the obstacle surface $\bm{\Upsilon}(\bm{p};\bm{p}_o)=1$ is equivalent to $\bm{F}_{OA}|_{\bm{\Upsilon}=1}=\bm{\tau}_{\bm{\Upsilon}}^b$. Then, it follows that
  \begin{equation}\label{eq_n_dot_tau_b}
    (\bm{F}_{OA}\cdot\bm{n}_{\bm{\Upsilon}})|_{\bm{\Upsilon}=1}=\bm{\tau}_{\bm{\Upsilon}}^b\cdot\bm{n}_{\bm{\Upsilon}}.
  \end{equation}
  For the case of $\bm{n}_{\bm{\Upsilon}}\nparallel\bm{F}$, we substitute (\ref{tau_Upsilon_2_case1}) into (\ref{eq_n_dot_tau_b}), and it is obtained that
  \begin{equation}\label{eq_n_dot_tau_b_zero_case1}
    (\bm{F}_{OA}\cdot\bm{n}_{\bm{\Upsilon}})|_{\bm{\Upsilon}=1}=(\bm{\tau}_{\bm{\Upsilon}}^a\times\bm{n}_{\bm{\Upsilon}})\cdot\bm{n}_{\bm{\Upsilon}}
    =-(\bm{n}_{\bm{\Upsilon}}\times\bm{n}_{\bm{\Upsilon}})\cdot\bm{\tau}_{\bm{\Upsilon}}^a=0,
  \end{equation}
  where the property of mixed product is utilized. For the case of $\bm{n}_{\bm{\Upsilon}}\parallel\bm{F}$, we have
  \begin{equation}\label{eq_n_dot_tau_b_zero_case2}
    (\bm{F}_{OA}\cdot\bm{n}_{\bm{\Upsilon}})|_{\bm{\Upsilon}=1}=\bm{\tau}_{\bm{\Upsilon}}^b\cdot\bm{n}_{\bm{\Upsilon}}=\bm{G}\cdot\bm{n}_{\bm{\Upsilon}},
  \end{equation}
  where (\ref{tau_Upsilon_2_case2}) is employed. Due to $\bm{n}_{\bm{\Upsilon}}\parallel\bm{F}$, the normal vector $\bm{n}_{\bm{\Upsilon}}$ can be expressed as $\bm{n}_{\bm{\Upsilon}}=k\bm{F}$, where $k$ is a nonzero scalar. By substitute it into (\ref{eq_n_dot_tau_b_zero_case2}), we have
  \begin{equation}\label{eq_n_dot_tau_b_zero_case2_end}
    (\bm{F}_{OA}\cdot\bm{n}_{\bm{\Upsilon}})|_{\bm{\Upsilon}=1}=k\bm{G}\cdot\bm{F}=0,
  \end{equation}
  which is guaranteed by (\ref{eq_F_cdot_G}). Equations (\ref{eq_n_dot_tau_b_zero_case1}) and (\ref{eq_n_dot_tau_b_zero_case2_end}) indicate that the OA-VF $\bm{F}_{OA}$ is orthogonal to the normal vector $\bm{n}_{\bm{\Upsilon}}$ on the obstacle surface. In other words, the $\bm{F}_{OA}$ lies in the tangential plane of the obstacle surface, so that it will never penetrate the obstacle area $\bm{\mathcal{A}}_o$.

  2) We first consider the case of $\bm{n}_{\bm{\Upsilon}}\nparallel\bm{F}$. For simplicity, we use the components to represent the VFs $\bm{n}_{\bm{\Upsilon}}$ and $\bm{F}$ in the following. To be more specific, they are denoted by $\bm{n}_{\bm{\Upsilon}}=[\frac{\partial \bm{\Upsilon}}{\partial x}\ \frac{\partial \bm{\Upsilon}}{\partial y}\ \frac{\partial \bm{\Upsilon}}{\partial z}]^{\rm T}$ and $\bm{F}=[F_x\ F_y\ F_z]^{\rm T}$, respectively. Under such a formulation, the cross product ``$\times$" can be expressed as the hat map ``$\wedge$" given in (\ref{eq_hat_map}). Then, $\bm{\tau}_{\bm{\Upsilon}}^a$ in (\ref{tau_Upsilon_1}) can be rewritten as $\bm{\tau}_{\bm{\Upsilon}}^a=\bm{n}_{\bm{\Upsilon}}^{\wedge}\bm{F}$, and $\bm{\tau}_{\bm{\Upsilon}}^b$ in (\ref{tau_Upsilon_2_case1}) can be rewritten as
  \begin{equation}\label{eq_tau_b_with_hat}
    \bm{\tau}_{\bm{\Upsilon}}^b=(\bm{n}_{\bm{\Upsilon}}^{\wedge}\bm{F})\times\bm{n}_{\bm{\Upsilon}} = -\bm{n}_{\bm{\Upsilon}}^{\wedge}(\bm{n}_{\bm{\Upsilon}}^{\wedge}\bm{F}) = -(\bm{n}_{\bm{\Upsilon}}^{\wedge})^2\bm{F}.
  \end{equation}
  By substituting (\ref{eq_tau_b_with_hat}) into (\ref{eq_F_ob}), $\bm{F}_{OA}$ can be reorganized as
  \begin{equation}\label{eq_F_OA_with_hat}
    \bm{F}_{OA} = \chi\bm{F}+(\chi-1)(\bm{n}_{\bm{\Upsilon}}^{\wedge})^2\bm{F} = \bm{\Xi}_{OA}\bm{F},
  \end{equation}
  where $\bm{\Xi}_{OA}$ is a matrix defined by
  \begin{equation}\label{eq_Xi_OA}
    \bm{\Xi}_{OA}=\chi\bm{I}+(\chi-1)(\bm{n}_{\bm{\Upsilon}}^{\wedge})^2,
  \end{equation}
  and $\bm{I}$ is the identity matrix. Based on the formulations given in (\ref{eq_VF_xyz_compo}), $\bm{F}$ only vanishes as the desired point $\bm{p}_d$. According to Assumption~\ref{assum_p_d}, there does not exist $\bm{p}_d$ in the reactive area $\bm{\mathcal{A}}_r$, implying that $\bm{F}\ne\bm{0}$ in $\bm{\mathcal{A}}_r$. Therefore, referring to (\ref{eq_F_OA_with_hat}), $\bm{F}_{OA}$ does not vanish in $\bm{\mathcal{A}}_r$ if and only if the matrix $\bm{\Xi}_{OA}$ is invertible. By calculation, the determinant of $\bm{\Xi}_{OA}$ is
  \begin{equation}
    \verb"det"(\bm{\Xi}_{OA})= \chi+2\chi^2(1-\chi)\|\bm{n}_{\bm{\Upsilon}}\|^2+\chi(1-\chi)^2\|\bm{n}_{\bm{\Upsilon}}\|^4.
  \end{equation}
  Due to $\|\bm{n}_{\bm{\Upsilon}}\|\ne 0$, it is obvious that $\verb"det"(\bm{\Xi}_{OA})\ne 0$ for $0<\chi\le 1$, indicating that the matrix $\bm{\Xi}_{OA}$ is invertible in $\bm{\mathcal{A}}_r/\{\bm{\Upsilon}(\bm{p};\bm{p}_o)=1\}$. Once $\bm{p}\in\{\bm{\Upsilon}(\bm{p};\bm{p}_o)=1\}$, i.e., $\chi=0$, there hold $\verb"det"(\bm{\Xi}_{OA})=0$, implying the matrix $\bm{\Xi}_{OA}$ will lose rank. In this case, $\bm{F}_{OA}$ does not vanish if and only if $\bm{F}$ does not lie in $\verb"ker"(\bm{\Xi}_{OA})$. When $\chi=0$, $\bm{\Xi}_{OA}$ degenerates to
  \begin{equation}\label{eq_Xi_OA_de}
    \bm{\Xi}_{OA}=-(\bm{n}_{\bm{\Upsilon}}^{\wedge})^2.
  \end{equation}
  Thus, it can be derived that
  \begin{equation}
    \verb"ker"(\bm{\Xi}_{OA})={\rm ker}(\bm{n}_{\bm{\Upsilon}}^{\wedge})=\{\bm{\alpha}\in\mathbb{R}^3:\bm{\alpha}\parallel\bm{n}_{\bm{\Upsilon}}\}.
  \end{equation}
  Note that the premise of this case is $\bm{n}_{\bm{\Upsilon}}\nparallel\bm{F}$, so that $\bm{F}\notin\verb"ker"(\bm{\Xi}_{OA})$ indicating that $\bm{F}_{OA}\ne\bm{0}$ on the obstacle surface $\bm{\Upsilon}(\bm{p};\bm{p}_o)=1$. Therefore, $\bm{F}_{OA}$ does not vanish in the reactive area $\bm{\mathcal{A}}_r$ for the case of $\bm{n}_{\bm{\Upsilon}}\nparallel\bm{F}$.

  Next, we deal with the case of $\bm{n}_{\bm{\Upsilon}}\parallel\bm{F}$. In this case, owing to the formulation of $\bm{\tau}_{\bm{\Upsilon}}^b$ in (\ref{tau_Upsilon_2_case2}), $\bm{F}_{OA}$ in (\ref{eq_F_ob}) degenerates to
  \begin{equation}\label{eq_F_ob_de}
    \bm{F}_{OA}=\chi\bm{F}+(1-\chi)\bm{G}.
  \end{equation}
  According to Lemma~\ref{lem_VF_G}, the VFs $\bm{F}$ and $\bm{G}$ are linearly independent. Hence, as the linear combination of $\bm{F}$ and $\bm{G}$, the OA-VF $\bm{F}_{OA}$ does not vanish everywhere in the reactive area $\bm{\mathcal{A}}_r$.
\end{proof}

\begin{remark}
    One can observe the following two properties.
    \begin{enumerate}
        \item $\bm{F}$, $\bm{n_\Upsilon}$ and $\bm{\tau_\Upsilon}^b$ are always in the same plane since they are all orthogonal to $\bm{\tau_\Upsilon}^a$.
        \item The angle between $\bm{F}$ and $\bm{\tau_\Upsilon}^b$ is always less or equal to $\pi/2$, due to the fact that $\bm{F} ^{T}\bm{\tau_\Upsilon}^b \overset{\eqref{eq_tau_b_with_hat}}{=} -\bm{F}^T ({\bm{n_\Upsilon}^\wedge})^2 \bm{F} = \bm{F}^T {{\bm{n_\Upsilon}^\wedge}}^T {\bm{n_\Upsilon}^\wedge} \bm{F} = || {\bm{n_\Upsilon}^\wedge} \bm{F} ||^2 \ge 0$, where the equality holds only in the case of $\bm{n}_{\bm{\Upsilon}}\parallel\bm{F}$.
    \end{enumerate}
\end{remark}

Based on Theroem~\ref{theo_F_OA}, in order to avoid obstacles, the OA-VF $\bm{F}_{OA}$ can be employed as the nonholonomic robot's heading direction (i.e., the $x$-axis direction of the body-fixed frame) in the reactive area $\bm{\mathcal{A}}_r$. However, as shown in the free-space motion planning, it is necessary to design another two vector fields to determine the $y$-axis and $z$-axis directions of the body-fixed frame $\bm{\mathcal{F}}_{\rm e}$, and the requirement for these three VFs is that each two of them should be orthogonal. The following lemma gives the formulation of the other two VFs.
\begin{lemma}\label{lem_G_OA_H_OA}
  Define two VFs $\bm{G}_{OA}$ and $\bm{H}_{OA}$ by
  \begin{align}\label{eq_G_OA}
    \bm{G}_{OA} &= \mu\bm{G}+(1-\mu)\bm{n}_{\bm{\Upsilon}}, \\
    \bm{H}_{OA} &= \chi\mu\bm{H}+(1-\chi)(1-\mu)\bm{\tau}_{\Upsilon}^a + \mu(1-\chi)\bm{G}\times\bm{\tau}_{\bm{\Upsilon}}^b \nonumber \\
    &\quad + \chi(1-\mu)\bm{n}_{\bm{\Upsilon}}\times\bm{F}, \label{eq_H_OA}
  \end{align}
  where $\mu$ is given by
  \begin{equation}\label{eq_mu}
    \mu=\frac{\chi\bm{F}\cdot\bm{n}_{\bm{\Upsilon}}}{\chi\bm{F}\cdot\bm{n}_{\bm{\Upsilon}}+(\chi-1)\bm{\tau}_{\bm{\Upsilon}}^b\cdot\bm{G}},
  \end{equation}
  $\bm{G},\bm{H},\bm{n}_{\bm{\Upsilon}},\bm{\tau}_{\bm{\Upsilon}}^a,\chi$ are given in (\ref{eq_VF_G_xyz}), (\ref{eq_VF_H_xyz}), (\ref{n_Upsilon}), (\ref{tau_Upsilon_1}), (\ref{eq_chi}), respectively. Then, each two VFs of $\bm{F}_{OA},\bm{G}_{OA},\bm{H}_{OA}$ are orthogonal.
\end{lemma}

\begin{proof}
  We compute the inner product of $\bm{F}_{OA}$ and $\bm{G}_{OA}$ as given below
  \begin{align*}
    \bm{F}_{OA}\cdot\bm{G}_{OA} &= \chi\mu\bm{F}\cdot\bm{G}+(1-\chi)(1-\mu)\bm{\tau}_{\bm{\Upsilon}}^b\cdot\bm{n}_{\bm{\Upsilon}} \nonumber \\
    &\quad +\chi(1-\mu)\bm{F}\cdot\bm{n}_{\bm{\Upsilon}}+\mu(1-\chi)\bm{\tau}_{\bm{\Upsilon}}^b\cdot\bm{G}
  \end{align*}
  Due to $\bm{F}\cdot\bm{G}=0$ and $\bm{\tau}_{\bm{\Upsilon}}^b\cdot\bm{n}_{\bm{\Upsilon}}=0$, the formulation of $\bm{F}_{OA}\cdot\bm{G}_{OA}$ can be reorganized as
  \begin{equation}\label{eq_F_OA cdot_G_OA}
    \bm{F}_{OA}\cdot\bm{G}_{OA}=-\mu(\chi\bm{F}\cdot\bm{n}_{\bm{\Upsilon}}+(\chi-1)\bm{\tau}_{\bm{\Upsilon}}^b\cdot\bm{G}) +\chi\bm{F}\cdot\bm{n}_{\bm{\Upsilon}}.
  \end{equation}
  By substituting (\ref{eq_mu}) into (\ref{eq_F_OA cdot_G_OA}), we have $\bm{F}_{OA}\cdot\bm{G}_{OA}=0$, indicating that the vector fields $\bm{F}_{OA}$ and $\bm{G}_{OA}$ are orthogonal.

  Next, we compute the cross product of $\bm{G}_{OA}$ and $\bm{F}_{OA}$.
  \begin{align*}
    \bm{G}_{OA}\times\bm{F}_{OA} &=\chi\mu\bm{G}\times\bm{F}+(1-\chi)(1-\mu)\bm{n}_{\bm{\Upsilon}}\times\bm{\tau}_{\bm{\Upsilon}}^b \nonumber \\
    &\quad +\mu(1-\chi)\bm{G}\times\bm{\tau}_{\bm{\Upsilon}}^b+\chi(1-\mu)\bm{n}_{\bm{\Upsilon}}\times\bm{F}.
  \end{align*}
  Owing to $\bm{G}\times\bm{F}=\bm{H}$ and $\bm{n}_{\bm{\Upsilon}}\times\bm{\tau}_{\bm{\Upsilon}}^b=\bm{\tau}_{\bm{\Upsilon}}^a$, the formulation of $\bm{G}_{OA}\times\bm{F}_{OA}$ can be further simplified to be (\ref{eq_H_OA}), that is, $\bm{H}_{OA}=\bm{G}_{OA}\times\bm{F}_{OA}$, implying that $\bm{H}_{OA}$ is orthogonal to $\bm{F}_{OA}$ and $\bm{G}_{OA}$, respectively.
\end{proof}

With the help of Lemma~\ref{lem_G_OA_H_OA}, we are able to construct an auxiliary attitude matrix $\bm{R}_a$ according to the VFs $\bm{F}_{OA},\bm{G}_{OA},\bm{H}_{OA}$. To be more specific, by replacing the components of $\bm{F},\bm{G},\bm{H}$ in (\ref{eq_auxi_atti_matrix}) with those of $\bm{F}_{OA}$, $\bm{G}_{OA}$, $\bm{H}_{OA}$, we can obtain an attitude matrix $\bm{R}_a$ for obstacle avoidance. Then, based on $\bm{R}_a$, the attitude tracking control input can be designed as (\ref{eq_Omega_design}).

Although the OA-VF $\bm{F}_{OA}$ (\ref{eq_F_ob}) is proposed for one obstacle, it can be easily extended to the situation of multiple obstacles. Assume there exist $M$ obstacles in the workspace and each one can be described by $\bm{\Upsilon}_i(\bm{p};\bm{p}_{oi})=1$, $i=1,\cdots,M$. Then, the obstacle avoidance vector field can be specified as
\begin{equation}\label{eq_F_OA_multi_obs}
  \bm{F}_{OA}=\left(\prod_{i=1}^{M}\chi_i\right)\bm{F}+\sum_{i=1}^{M}\left((1-\chi_i)\bm{\tau}_{\bm{\Upsilon}i}^b\right),
\end{equation}
where $\chi_i$ is the transition function of the $i$th obstacle, $\bm{F}$ is the N-VF of the nonholonomic robot, and $\bm{\tau}_{\bm{\Upsilon}i}^b$ is the tangential VF of the $i$th obstacle as designed in (\ref{tau_Upsilon_2_case1}) and (\ref{tau_Upsilon_2_case2}).

\section{Cooperative Motion Planning}\label{sec_CA_VF}

This section investigates the cooperative motion planning of multiple nonholonomic robots via VFs, where the main problem is how to achieve collision avoidance among nonholonomic robots. Motivated by the obstacle avoidance results in the previous section, we regard other nonholonomic robots to be avoided as moving obstacles, and introduce different priorities of the nonholonomic robots to deal with the challenge of motion coupling.

By comparing (\ref{eq_C_ai})-(\ref{eq_C_ri}) with (\ref{eq_A_o})-(\ref{eq_A_r}), we observe that the setting of collision avoidance can be regarded as avoiding moving obstacles with sphere boundaries. However, the motion couplings among nonholonomic robots make the collision avoidance essentially different from the obstacle avoidance. To be more specific, the movement of the obstacle is independent of the nonholonomic robot, while in the collision avoidance, the motion of nonholonomic robots relies on each other indeed. Let us take two nonholonomic robots labelled as $A$ and $B$ for example. Once $A$ enters the reactive area of $B$, it would intend to circumvent $B$ for non-collision. In the meantime, $B$ has also detected $A$ and would take actions to avoid $A$ as well. Such a motion coupling introduces an algebraic loop in mathematics. Thus, compared with obstacle avoidance, how to realize motion decoupling is the main challenge in collision avoidance of multiple nonholonomic robots.

In this paper, we specify a particular priority for each nonholonomic robot to deal with the motion coupling. Specifically, nonholonomic robots are labelled by $i\in\mathcal{I}_\mathcal{V}=\{1,\cdots,N\}$, and it is prescribed that the nonholonomic robot with a smaller index has higher priority. That is to say, the $1$-st nonholonomic robot possesses the highest priority, while the $N$-th is the lowest prioritized. Regarding the $i$-th nonholonomic robot, we define the neighboring label set
\begin{equation}
  \bm{\mathcal{N}}_i=\{j\in\mathcal{I}_\mathcal{V}:\bm{p}_i\in\bm{\mathcal{C}}_{rj},i\ne j\}.
\end{equation}
The elements in $\bm{\mathcal{N}}_i$ are the labels of the nonholonomic robots whose reactive area is entered by the $i$-th nonholonomic robot. In other words, the nonholonomic robots detected by $i$-th are collected in $\bm{\mathcal{N}}_i$. Next, we select certain labels in $\bm{\mathcal{N}}_i$ to define another set
\begin{equation}
  \bm{\mathcal{N}}_i^+=\{j\in\bm{\mathcal{N}}_i:j<i\},
\end{equation}
which collects the nonholonomic robots that are detected by the $i$-th robot and have higher priorities than the $i$-th robot. Then, in order to realize collision avoidance, we let the $i$-th nonholonomic robot avoid the nonholonomic robots in $\bm{\mathcal{N}}_i^+$, while ignore those in $\bm{\mathcal{N}}_i/\bm{\mathcal{N}}_i^+$. In this way, the motion of the high prioritized nonholonomic robots is independent from the low prioritized ones, so that the problem of collision avoidance among robots can be converted into the problem of moving obstacles avoidance.

\begin{figure}
  \centering
  \includegraphics[width=0.42\textwidth,trim=0 5 0 0,clip]{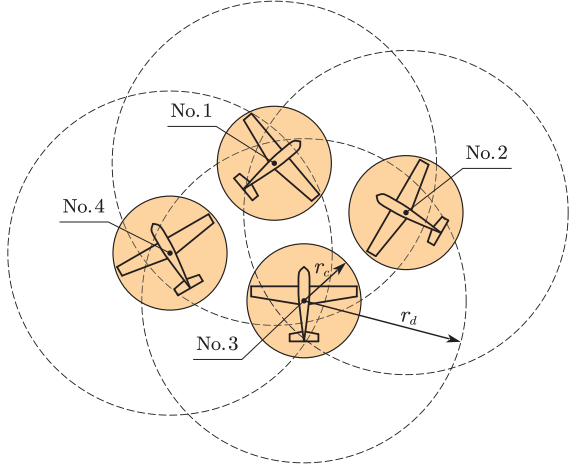}
  \caption{Collision avoidance of multiple nonholonomic robots (taking fixed-wing UAVs for example).}
  \label{fig_four_UAVs}
\end{figure}

For example, as shown in Figure~\ref{fig_four_UAVs}, there are four nonholonomic robots in the workspace, and the sets $\bm{\mathcal{N}}_i$ and $\bm{\mathcal{N}}_i^+$ ($i\in\mathcal{I}_\mathcal{V}=\{1,2,3,4\}$) can be obtained as
\begin{align*}
  &\bm{\mathcal{N}}_1=\{2,3,4\},\quad \bm{\mathcal{N}}_1^+=\varnothing, \\
  &\bm{\mathcal{N}}_2=\{1,3\},\qquad \bm{\mathcal{N}}_2^+=\{1\}, \\
  &\bm{\mathcal{N}}_3=\{1,2,4\},\quad \bm{\mathcal{N}}_3^+=\{1,2\}, \\
  &\bm{\mathcal{N}}_4=\{1,3\},\qquad \bm{\mathcal{N}}_4^+=\{1,3\}.
\end{align*}
Then, the labels in $\bm{\mathcal{N}}_i^+$ are the nonholonomic robots that the $i$-th nonholonomic robot should circumvent, by regarding them as independent moving obstacles.

Therefore, similar to (\ref{n_Upsilon}), we derive the normal vector of the sphere $\bm{\Psi}_j(\bm{p};\bm{p}_j)=\bar{k}$ $(r_c\le\bar{k}\le r_d)$ as follows
\begin{equation}\label{n_Psi}
  \bm{n}_{\bm{\Psi}j}=\frac{\partial \bm{\Psi}_j}{\partial x}\bm{e}_x+\frac{\partial \bm{\Psi}_j}{\partial y}\bm{e}_y+\frac{\partial \bm{\Psi}_j}{\partial z}\bm{e}_z,
\end{equation}
and the time derivative of $\bm{n}_{\bm{\Psi}j}$ is
\begin{equation*}
  \dot{\bm{n}}_{\bm{\Psi}j}=\frac{\partial \bm{n}_{\bm{\Psi}_j}}{\partial \bm{p}}\dot{\bm{p}}+\frac{\partial \bm{n}_{\bm{\Psi}_j}}{\partial \bm{p}_j}\dot{\bm{p}_j},
\end{equation*}
where $j\in\bm{\mathcal{N}}_i^+$. Then, the collision avoidance VF (CA-VF) for the $i$-th nonholonomic robot can be designed by
\begin{equation}\label{eq_F_CA_i}
  \bm{F}_{CAi}=\left(\prod_{j\in\bm{\mathcal{N}}_i^+}\chi_j\right)\bm{F}_i+\sum_{j\in\bm{\mathcal{N}}_i^+}\left((1-\chi_j)\bm{\tau}_{\bm{\Psi}j}^b\right),
\end{equation}
where $\bm{F}_i$ is the N-VF of the $i$-th nonholonomic robot in the free space, $\bm{\tau}_{\bm{\Psi}j}^b$ is the tangential VF given by
\begin{equation}\label{eq_tau_psi_b}
  \bm{\tau}_{\bm{\Psi}j}^b=\left\{
  \begin{aligned}
    (\bm{n}_{\bm{\Psi}j}\times\bm{F}_i)\times\bm{n}_{\bm{\Psi}j},\quad  &{\rm if}\ \bm{n}_{\bm{\Psi}j}\nparallel\bm{F}_i, \\
    \bm{G}_i,\qquad\qquad  &{\rm if}\ \bm{n}_{\bm{\Psi}j}\parallel\bm{F}_i,
  \end{aligned}
  \right.
\end{equation}
and $\chi_j$ is the transition function given by
\begin{equation}\label{eq_chi_j}
  \chi_j=\left\{
  \begin{aligned}
    0,\quad     \ &\bm{p}\in\{\bm{\Psi}_j(\bm{p};\bm{p}_j)=r_c\}, \\
    S(\bm{p}),\ \ &\bm{p}\in\{r_c<\bm{\Psi}_j(\bm{p};\bm{p}_j)<r_d\}, \\
    1,\quad     \ &\bm{p}\in\{\bm{\Psi}_j(\bm{p};\bm{p}_j)=r_d\}.
  \end{aligned}
  \right.
\end{equation}

Algorithm~\ref{alg} summarizes the main procedure to design the CA-VF $\bm{F}_{CAi}$ of the $i$-th nonholonomic robot.

\begin{algorithm}[tbp]
  \caption{Design of the CA-VF $\bm{F}_{CAi}$}
  \label{alg}
  \KwIn{$\bm{p}_i$ (the present position of the $i$-th robot)}
  \KwOut{$\bm{F}_{CAi}$ (the CA-VF of the $i$-th robot)}
  Determine the detection range $\bm{\mathcal{D}}_i$ of the $i$-th robot

  Detect other robots' position $\bm{p}_j$ within $\bm{\mathcal{D}}_i$

  \For{each $\bm{p}_j\in\bm{\mathcal{D}}_i$}{

  \If{$\bm{p}_i\in\bm{\mathcal{C}}_{rj}$}{

  Add $j$ to the set $\bm{\mathcal{N}}_i$

  \If{$j<i$}{

  Add $j$ to the set $\bm{\mathcal{N}}_i^+$
  }}}

  \For{each $j\in\bm{\mathcal{N}}_i^+$}{

  Compute the vector field $\bm{n}_{\bm{\Psi}j}$

  \eIf{$\bm{n}_{\bm{\Psi}j}\nparallel\bm{F}_i$}{

  $\bm{\tau}_{\bm{\Psi}j}^b=(\bm{n}_{\bm{\Psi}j}\times\bm{F}_i)\times\bm{n}_{\bm{\Psi}j}$}{
  $\bm{\tau}_{\bm{\Psi}j}^b=\bm{G}_j$
  }

  Compute the transition function $\chi_j$
  }

  Compute the CA-VF $\bm{F}_{CAi}$ as (\ref{eq_F_CA_i})
\end{algorithm}

Having designed $\bm{F}_{CAi}$, we can simply construct the auxiliary VFs $\bm{G}_{CAi}$ and $\bm{H}_{CAi}$ as in Lemma~\ref{lem_G_OA_H_OA}. Subsequently, based on $\bm{F}_{CAi}$, $\bm{G}_{CAi}$, $\bm{H}_{CAi}$, the auxiliary attitude matrix $\bm{R}_a$ can be constructed as (\ref{eq_auxi_atti_matrix}), and then the attitude tracking control input $\bm{\Omega}^{\wedge}$ can be designed accordingly by Lemma~\ref{lem_R_e_stabilize}.

\section{Numerical Simulation Results}\label{sec_sim}

In this section, we provide three numerical simulation examples based on fixed-wing UAVs to verify the effectiveness of the proposed motion planning algorithms.

The first example is to drive a fixed-wing UAV to the target position with the desired heading direction in an obstacle-free environment. The target position is chosen as the origin of $\bm{\mathcal{F}}_{xyz}$ and the desired heading direction is specified as the positive direction of the $x$-axis, that is, $\bm{p}_d=[0\ \ 0\ \ 0]^{\rm T}$ and $\bm{e}_d=\bm{e}_x=[1\ \ 0\ \ 0]^{\rm T}$. The simulation is carried out based on six initial conditions, including the position and orientation of the fixed-wing UAV. Figure~\ref{fig_one_UAV_tra} depicts the trajectories of the fixed-wing UAV from different initial conditions, demonstrating that the fixed-wing UAV reaches the goal point with the specified heading direction. The position and orientation of the fixed-wing UAV evolving with time are shown in Figure~\ref{fig_one_UAV_curve}, where the orientation is parameterized as the Euler angles for the sake of illustration. Particularly, as seen in Figure~\ref{fig_one_UAV_curve}, the pitch and yaw angles converge to zero, indicating that the heading direction of the fixed-wing UAV points to the $x$-axis of $\bm{\mathcal{F}}_{xyz}$.

\begin{figure}[t]
  \centering
  \includegraphics[width=0.42\textwidth,trim=90 10 90 10,clip]{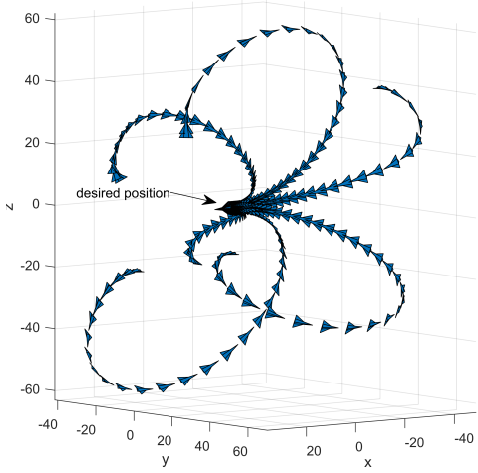}
  \caption{Trajectories of a fixed-wing UAV moving in an obstacle-free environment (with six different initial conditions).}
  \label{fig_one_UAV_tra}
\end{figure}

\begin{figure}[t]
  \centering
  \includegraphics[width=0.48\textwidth,trim=70 10 70 10,clip]{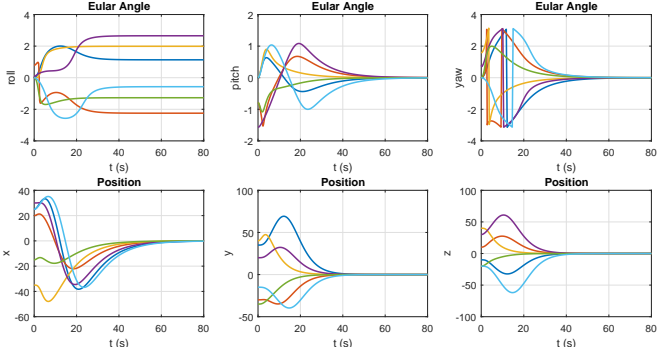}
  \caption{Position and orientation evolution of the fixed-wing UAV.}
  \label{fig_one_UAV_curve}
\end{figure}

In the second example, the fixed-wing UAV moves in an obstacle-cluttered environment. By following the obstacle description given in \cite{Khatib1986Khatib}, the obstacles in this simulation example are formulated by
\begin{equation*}
  \bm{\Upsilon}(\bm{p};\bm{p}_o)=\left(\frac{x-x_o}{a}\right)^{2p}+\left(\frac{y-y_o}{b}\right)^{2q}+\left(\frac{z-z_o}{c}\right)^{2r}=1,
\end{equation*}
where $p,q,r$ are used to describe the geometrical shape of the obstacle, and $a,b,c$ are used to control the size of the obstacle. By choosing different $p,q,r$, the equation $\bm{\Upsilon}(\bm{p};\bm{p}_o)=1$ represents the ellipsoid, cylinder, cone, etc. Simulation results based on three different initial conditions are given in Figure~\ref{fig_one_UAV_ob}, where the target position is chosen as $\bm{p}_d=[75\ \ 30\ \ 25]^{\rm T}$, and the desired heading direction is specified as $\bm{e}_d=\bm{e}_x=[1\ \ 0\ \ 0]^{\rm T}$.

\begin{figure*}[t]
  \centering
  \begin{minipage}[b]{0.62\textwidth}
    \subfigure[Trajectories in 3D space.]{
      \includegraphics[width=1\textwidth,trim=200 30 180 20,clip]{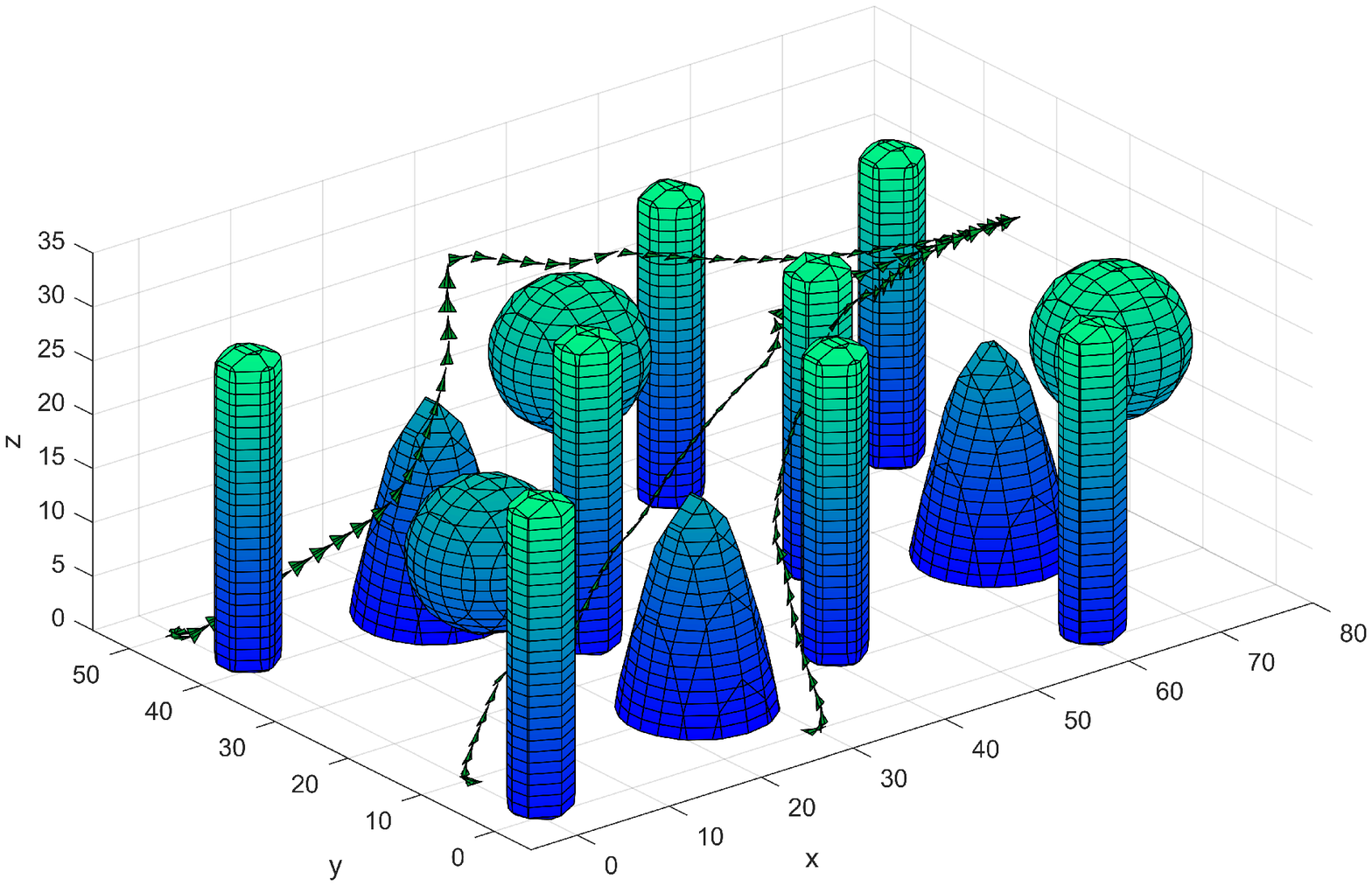}
      \label{fig_one_UAV_ob_0}}
  \end{minipage}
  \begin{minipage}[b]{0.35\textwidth}
    \subfigure[Trajectories in $x-z$ plane]{
      \includegraphics[width=0.95\textwidth,trim=80 0 80 0,clip]{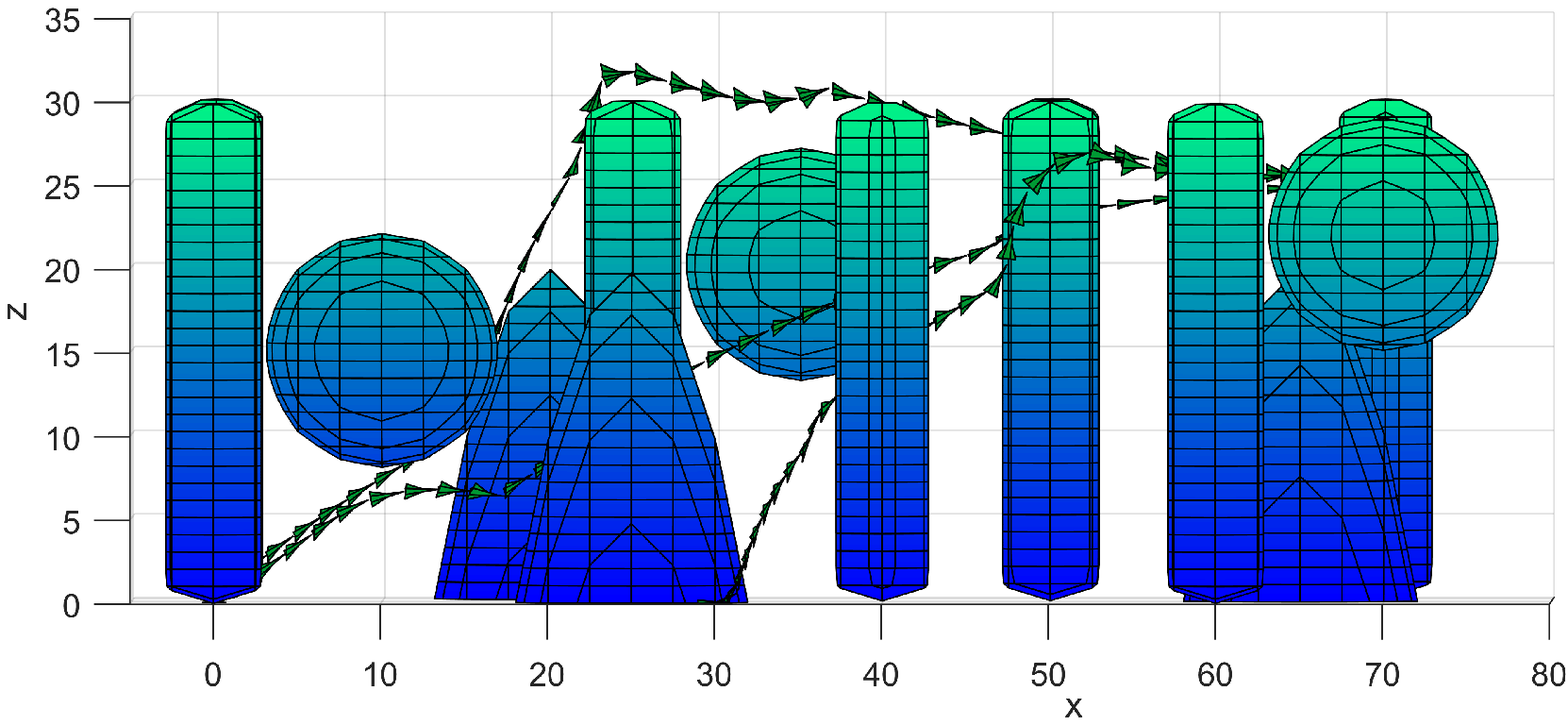}
      \label{fig_one_UAV_ob_1}} \\
    \subfigure[Trajectories in $y-z$ plane]{
      \includegraphics[width=0.95\textwidth,trim=180 0 180 0,clip]{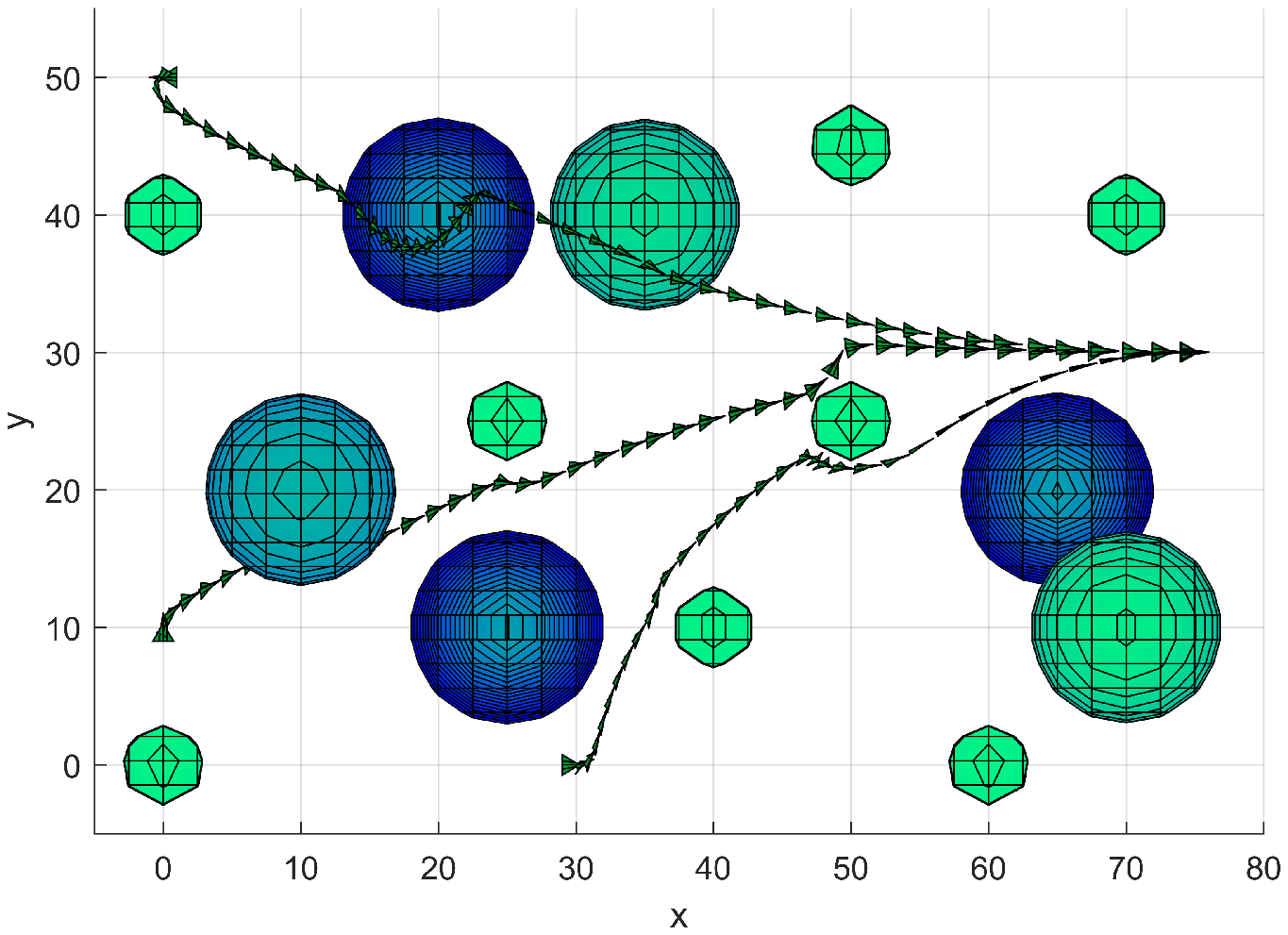}
      \label{fig_one_UAV_ob_2}}
  \end{minipage}
  \caption{Trajectories of a fixed-wing UAV moving in an obstacle-cluttered environment (three different initial conditions).}
  \label{fig_one_UAV_ob}
\end{figure*}

The third example provides the motion planning results of seven fixed-wing UAVs with collision avoidance. As shown in Figure~\ref{fig_seven_UAVs_CA}, each fixed-wing UAV is required to reach its desired position and keep the final heading direction as $\bm{e}_d=[\frac{\sqrt{2}}{2}\ \ \frac{\sqrt{2}}{2}\ \ 0]^{\rm T}$. It can be observed that the fixed-wing UAVs achieve the goal of motion planning and do not collide with each other. Besides, we note that the fixed-wing UAV in the middle of the line (which is marked in dark blue) is labelled by $i=1$, so that it has the highest priority and its movement is not influenced by others. In contrast, the fixed-wing UAV starting from the position $(0,0,0)$ (which is marked in dark magenta) is the lowest prioritized, and thus it has to avoid the rest of all fixed-wing UAVs.


\begin{figure*}[t]
  \centering
  \begin{minipage}[b]{0.58\textwidth}
    \subfigure[Trajectories in 3D space.]{
      \includegraphics[width=0.9\textwidth,trim=80 40 50 50,clip]{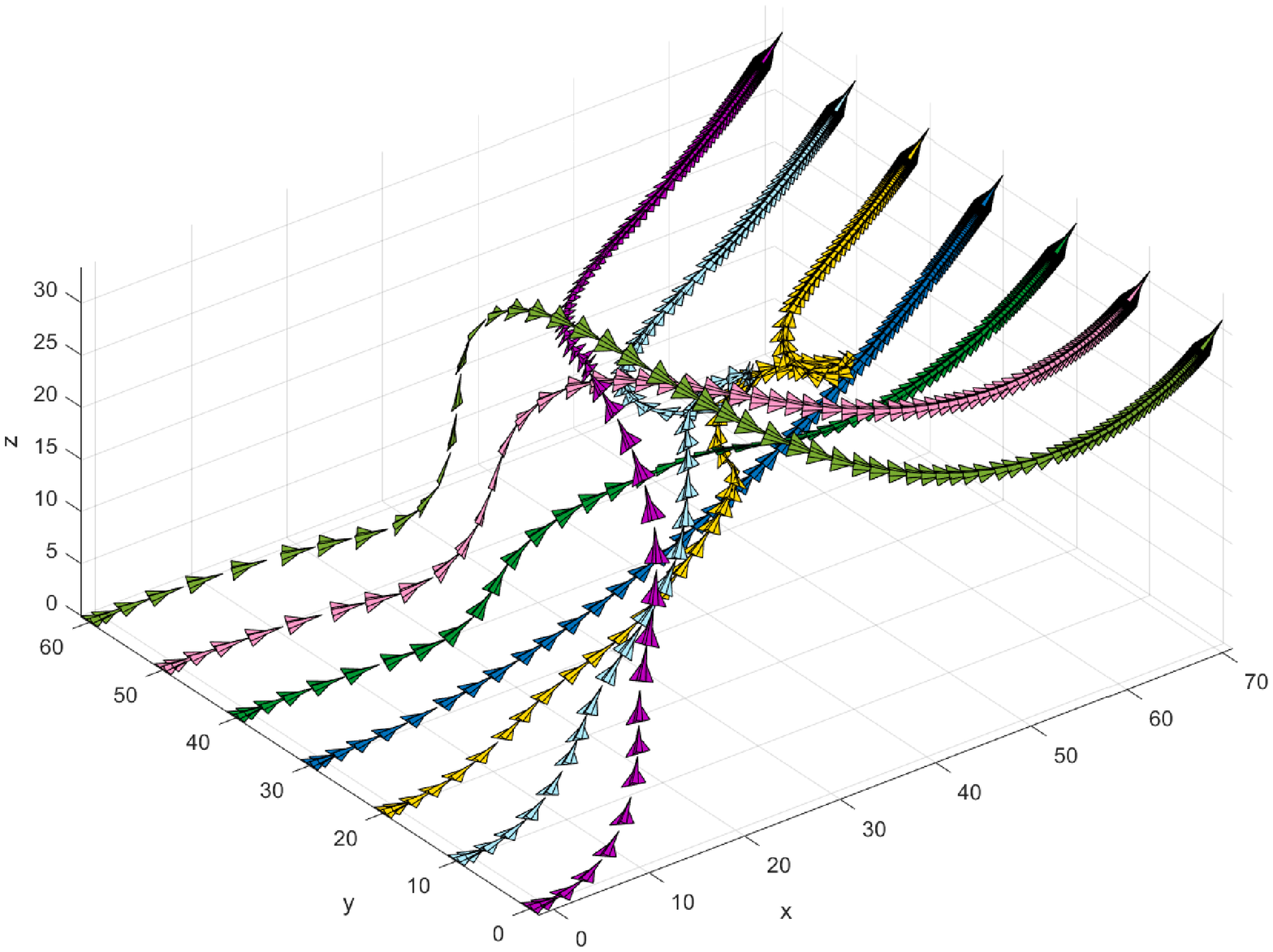}
      \label{fig_seven_UAVs_CA_0}}
  \end{minipage}
  \begin{minipage}[b]{0.4\textwidth}
    \subfigure[Trajectories in $x-z$ plane]{
      \includegraphics[width=0.7\textwidth,trim=50 80 50 80,clip]{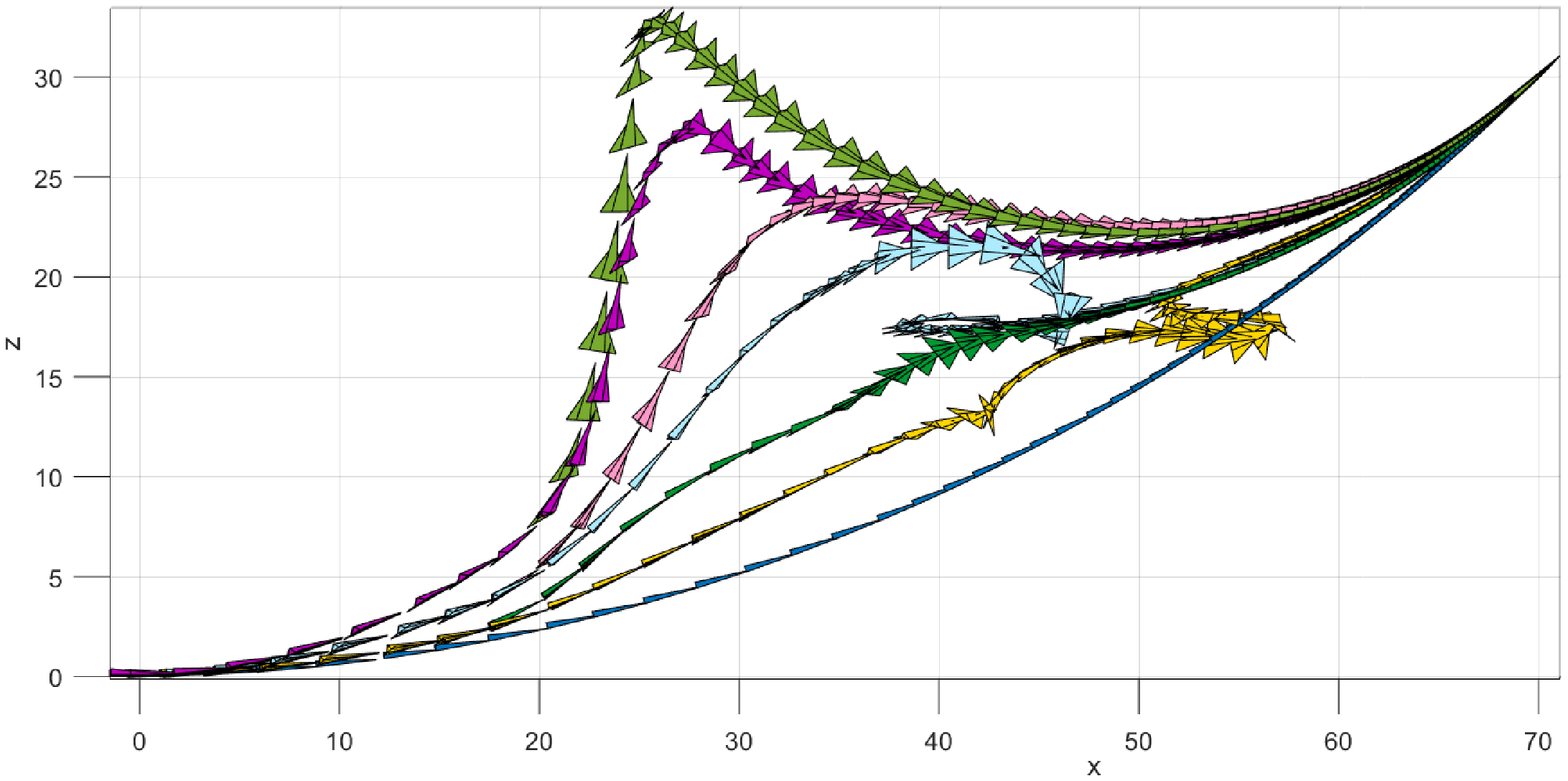}
      \label{fig_seven_UAVs_CA_1}} \\
    \subfigure[Trajectories in $y-z$ plane]{
      \includegraphics[width=0.7\textwidth,trim=50 80 50 80,clip]{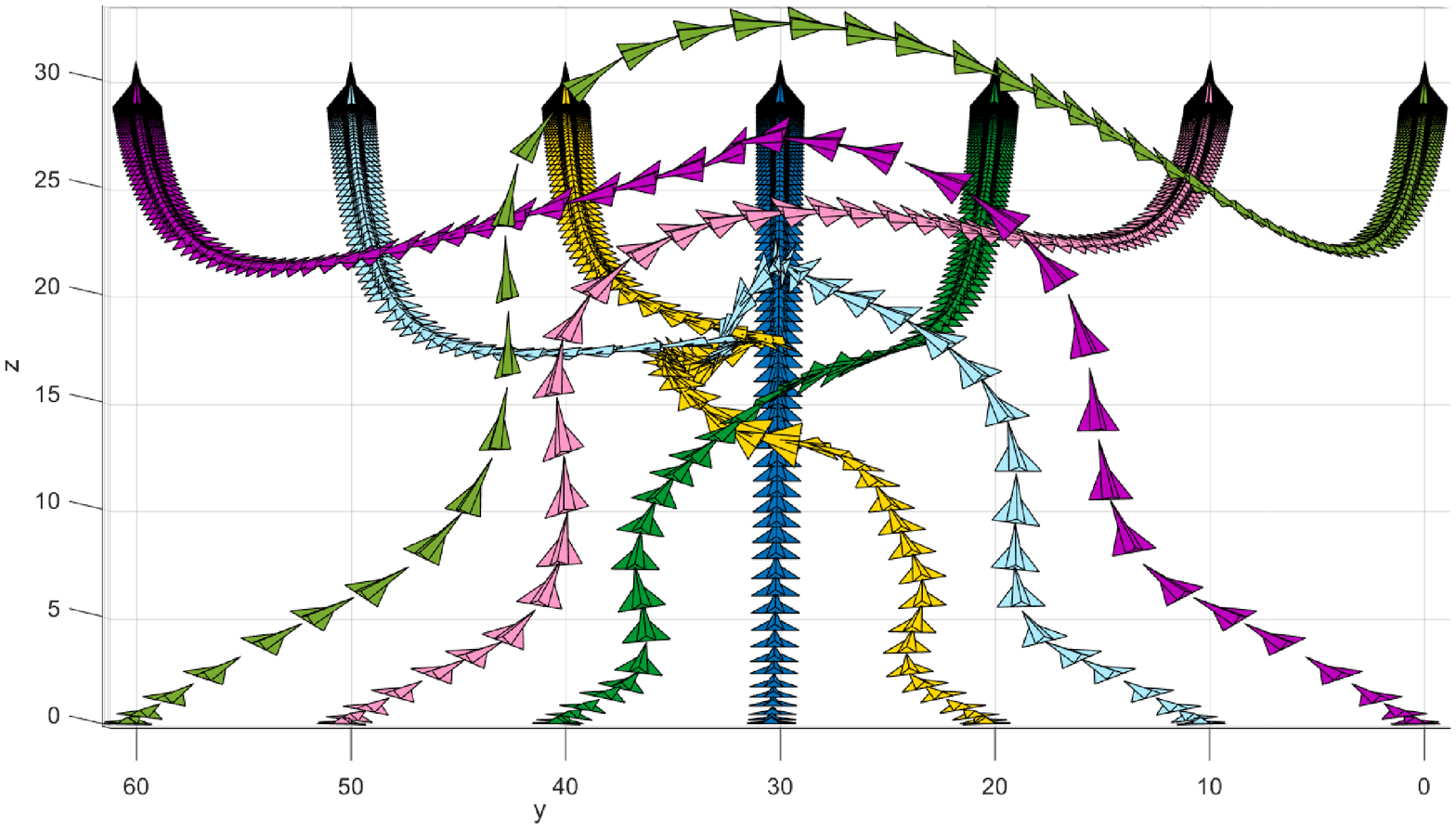}
      \label{fig_seven_UAVs_CA_2}}
  \end{minipage}
  \caption{Trajectories of seven fixed-wing UAVs with collision avoidance.}
  \label{fig_seven_UAVs_CA}
\end{figure*}

\section{Conclusion} \label{sec_conclu}

This paper has proposed a novel velocity vector field for nonholonomic robots in 3D to solve the motion planning problem where robots are required to not only reach the specified target positions but also align with the predefined heading directions. A composite vector field has been further presented to ensure collision avoidance with obstacles and other robots by guaranteeing no penetration of the dangerous area.  In addition, we have proposed a priority-based algorithm to achieve   motion decoupling among multiple nonholonomic robots. Future research will focus on the motion planning problem in more practical scenarios, such as input saturation constraints, path curvature constraints and robust motion with measurement noises.

\appendix

\section{Proof of Lemma~\ref{lem_kine_R_a}}\label{app_A}

By substituting (\ref{eq_auxi_atti_matrix}) and (\ref{eq_dot_auxi_atti_matrix}) into the definition of $\bm{\Omega}_a^\wedge$, we have
    \begin{equation*}
        \bm{\Omega}_a^\wedge=
        \begin{bmatrix}
            \frac{\bm{\zeta}_X^{\rm T}}{\|\bm{\zeta}_X\|} \frac{\rm d}{{\rm d}t}\left(\frac{\bm{\zeta}_X}{\|\bm{\zeta}_X\|}\right) & \frac{\bm{\zeta}_X^{\rm T}}{\|\bm{\zeta}_X\|} \frac{\rm d}{{\rm d}t}\left(\frac{\bm{\zeta}_Y}{\|\bm{\zeta}_Y\|}\right) & \frac{\bm{\zeta}_X^{\rm T}}{\|\bm{\zeta}_X\|} \frac{\rm d}{{\rm d}t}\left(\frac{\bm{\zeta}_Z}{\|\bm{\zeta}_Z\|}\right) \\
            \frac{\bm{\zeta}_Y^{\rm T}}{\|\bm{\zeta}_Y\|} \frac{\rm d}{{\rm d}t}\left(\frac{\bm{\zeta}_X}{\|\bm{\zeta}_X\|}\right) & \frac{\bm{\zeta}_Y^{\rm T}}{\|\bm{\zeta}_Y\|} \frac{\rm d}{{\rm d}t}\left(\frac{\bm{\zeta}_Y}{\|\bm{\zeta}_Y\|}\right) & \frac{\bm{\zeta}_Y^{\rm T}}{\|\bm{\zeta}_Y\|} \frac{\rm d}{{\rm d}t}\left(\frac{\bm{\zeta}_Z}{\|\bm{\zeta}_Z\|}\right) \\
            \frac{\bm{\zeta}_Z^{\rm T}}{\|\bm{\zeta}_Z\|} \frac{\rm d}{{\rm d}t}\left(\frac{\bm{\zeta}_X}{\|\bm{\zeta}_X\|}\right) & \frac{\bm{\zeta}_Z^{\rm T}}{\|\bm{\zeta}_Z\|} \frac{\rm d}{{\rm d}t}\left(\frac{\bm{\zeta}_Y}{\|\bm{\zeta}_Y\|}\right) & \frac{\bm{\zeta}_Z^{\rm T}}{\|\bm{\zeta}_Z\|} \frac{\rm d}{{\rm d}t}\left(\frac{\bm{\zeta}_Z}{\|\bm{\zeta}_Z\|}\right)
        \end{bmatrix}
    \end{equation*}
    In order to prove $\bm{\Omega}_a^\wedge\in\mathfrak{so}(3)$, we have to verify that $\bm{\Omega}_a^\wedge$ is an anti-asymmetric matrix. For the sake of illustration, let $\bm{\Omega}_a^\wedge(i,j)$ denote the entry in $i$-th row and $j$-th column ($i,j=1,2,3$). Regarding the diagonal entries, we take $\bm{\Omega}_a^\wedge(1,1)$ for instance and it can be computed that
    \begin{align*}
        \bm{\Omega}_a^\wedge(1,1)
        &=\frac{\bm{\zeta}_X^{\rm T}}{\|\bm{\zeta}_X\|} \left(\frac{\dot{\bm{\zeta}}_X}{\|\bm{\zeta}_X\|}-(\bm{\zeta}_X^{\rm T}\dot{\bm{\zeta}_X})\frac{\bm{\zeta}_X}{\|\bm{\zeta}_X\|^3}\right) \\
        &=\frac{\bm{\zeta}_X^{\rm T}\dot{\bm{\zeta}_X}-\bm{\zeta}_X^{\rm T}\dot{\bm{\zeta}_X}}{\|\bm{\zeta}_X\|^2} =0,
    \end{align*}
    where \eqref{eq_dot_auxi_unit_vec} is employed. Similarly, we can also obtain that $\bm{\Omega}_a^\wedge(2,2)=0$ and $\bm{\Omega}_a^\wedge(3,3)=0$. Regarding the off-diagonal entries, there holds
    \begin{align*}
        \bm{\Omega}_a^\wedge(1,2)&=\frac{\bm{\zeta}_X^{\rm T}\dot{\bm{\zeta}}_Y}{\|\bm{\zeta}_X\|\|\bm{\zeta}_Y\|}-(\bm{\zeta}_Y^{\rm T}\dot{\bm{\zeta}_Y})\frac{\bm{\zeta}_X^{\rm T}\bm{\zeta}_Y}{\|\bm{\zeta}_X\|\|\bm{\zeta}_Y\|^3} \\
        \bm{\Omega}_a^\wedge(2,1)&=\frac{\bm{\zeta}_Y^{\rm T}\dot{\bm{\zeta}}_X}{\|\bm{\zeta}_Y\|\|\bm{\zeta}_X\|}-(\bm{\zeta}_X^{\rm T}\dot{\bm{\zeta}_X})\frac{\bm{\zeta}_Y^{\rm T}\bm{\zeta}_X}{\|\bm{\zeta}_Y\|\|\bm{\zeta}_X\|^3}
    \end{align*}
    According to \eqref{eq_zeta_XYZ}, $\bm{\zeta}_X$ and $\bm{\zeta}_Y$ are the components of $\bm{F}$ and $\bm{H}$, which are orthogonal to each other based on \eqref{eq_VF_H_xyz}. Thus, we have $\bm{\zeta}_X^{\rm T}\bm{\zeta}_Y=\bm{\zeta}_Y^{\rm T}\bm{\zeta}_X=0$, and further obtain that
    \begin{equation*}
        \bm{\Omega}_a^\wedge(1,2)+\bm{\Omega}_a^\wedge(2,1)=\frac{\bm{\zeta}_X^{\rm T}\dot{\bm{\zeta}}_Y+\bm{\zeta}_Y^{\rm T}\dot{\bm{\zeta}}_X}{\|\bm{\zeta}_X\|\|\bm{\zeta}_Y\|}=\frac{\frac{\rm d}{{\rm d}t}(\bm{\zeta}_X^{\rm T}\bm{\zeta}_Y)}{\|\bm{\zeta}_X\|\|\bm{\zeta}_Y\|}.
    \end{equation*}
    Based on $\bm{\zeta}_X^{\rm T}\bm{\zeta}_Y=0$, there holds $\frac{\rm d}{{\rm d}t}(\bm{\zeta}_X^{\rm T}\bm{\zeta}_Y)=0$, and it follows that
    \begin{equation*}
        \bm{\Omega}_a^\wedge(1,2)+\bm{\Omega}_a^\wedge(2,1)=0.
    \end{equation*}
    Similarly, we can obtain that
    \begin{align*}
        \bm{\Omega}_a^\wedge(1,3)+\bm{\Omega}_a^\wedge(3,1)=0, \\
        \bm{\Omega}_a^\wedge(2,3)+\bm{\Omega}_a^\wedge(3,2)=0.
    \end{align*}
    Hence,  $\bm{\Omega}_a^\wedge$ is an anti-asymmetric matrix, i.e., $\bm{\Omega}_a^\wedge\in\mathfrak{so}(3)$.

\ifCLASSOPTIONcaptionsoff
  \newpage
\fi



\bibliographystyle{IEEEtran}
\bibliography{IEEEabrv,mybibfile}
\end{document}